%% file: main.tex
\pdfoutput=1
\documentclass{article}

\usepackage{microtype}
\usepackage{graphicx}
\usepackage{subfigure}
\usepackage{booktabs} 
\usepackage[utf8]{inputenc} 
\usepackage[T1]{fontenc}    
\usepackage{hyperref}       
\usepackage{url}            
\usepackage{amsfonts}       
\usepackage{nicefrac}       
\usepackage{microtype}      
\usepackage{subfigure}
\usepackage{bbm}
\usepackage{multirow}
\usepackage{amssymb}
\usepackage{float}
\usepackage{xcolor}
\newcommand*{\fullref}[1]{\hyperref[{#1}]{\autoref*{#1} \nameref*{#1}}}
\usepackage{verbatim}
\usepackage{amsmath}
\usepackage{enumitem}
\usepackage{authblk}

\usepackage{mathtools}
\usepackage{enumitem}
\usepackage{authblk}
\usepackage[numbers,sort]{natbib}
\usepackage[margin=1in]{geometry}

\usepackage{ntheorem}
\allowdisplaybreaks[4]
\newtheorem{theorem}{Theorem}

\newtheorem{lemma}{Lemma}
\newtheorem{corollary}{Corollary}
\newtheorem*{proof}{Proof}
\newtheorem{assumption}{Assumption}
\newcommand{\tmix}{t_{\operatorname{mix}}}
\def\*#1{\boldsymbol{#1}}

\newcommand{\myparagraph}[1]{\textbf{#1}\hspace{0.5em}}

\title{MixML: A Unified Analysis of Weakly Consistent\\ Parallel Learning}

\author{Yucheng Lu\thanks{Corresponds to: yl2967@cornell.edu} }
\author{Jack Nash\thanks{Corresponds to: jrn79@cornell.edu} }
\author{Christopher De Sa\thanks{Corresponds to: cdesa@cs.cornell.edu} }
\affil{Department of Computer Science, Cornell\ University}
\date{}

\begin{document}

\maketitle

\input{./section/Abstract.tex}
\input{./section/Introduction.tex}

\input{./section/Related_work.tex}
\input{./section/Communication.tex}
\input{./section/MixML.tex}
\input{./section/principle.tex}
\input{./section/Application.tex}
\input{./section/Experiments.tex}
\input{./section/Conclusion.tex}

\bibliography{main}

\newpage
\begin{center}
{\huge\textbf{Supplementary Material}}
\end{center}

\appendix
\input{./section/appendix.tex}

\end{document}

%% file: section/Abstract.tex
\begin{abstract}
Parallelism is a ubiquitous method for accelerating machine learning algorithms.
However, theoretical analysis of parallel learning is usually done in an algorithm- and protocol-specific setting, giving little insight about how changes in the structure of communication could affect convergence.
In this paper we propose MixML, a general framework for analyzing convergence of weakly consistent parallel machine learning. 
Our framework includes: (1) a unified way of modeling the communication process among parallel workers; (2) a new parameter, the mixing time $t_{\text{mix}}$, that quantifies how the communication process affects convergence; and (3) a principled way of converting a convergence proof for a sequential algorithm into one for a parallel version that depends only on $t_{\text{mix}}$. 
We show MixML recovers and improves on known convergence bounds for asynchronous and/or decentralized versions of many algorithms, including SGD and AMSGrad.
Our experiments substantiate the theory and show the dependency of convergence on the underlying mixing time.
\end{abstract}

%% file: section/Introduction.tex
\section{Introduction}\label{introduction section}
Learning algorithms have shown promising results in training machine learning models with their sequential versions \cite{sutskever2013importance,ruder2016overview}.
Given their success, practitioners have developed methods to speed up their computation. One general method is to run them in a parallel or distributed setting, where multiple processors or workers locally evaluate optimization steps on models while periodically performing communication to exchange intermediate results such as gradients \cite{li2014communication,recht2011hogwild} or current model parameters \cite{lian2017can}.
The communication protocols among workers substantially affect the convergence rate on the learning \cite{alistarh2020elastic}. The most basic protocol is centralized synchronous communication \cite{chen2016revisiting}, which maintains perfect consistency among workers, i.e., workers always reach agreement on the parameters before stepping into the next iteration. Despite the consistency guarantee, the basic protocol usually incurs synchronization or centralization overhead when the system scales up. 
To address this, several lines of research introduce various protocols which maintain weaker consistency but allows algorithms converge at the same asymptotic rate as with perfect consistency. 
These protocols include: allowing workers to communicate asynchronously \cite{lian2015asynchronous,zhang2015staleness}; pipelining the communication signals \cite{li2018pipe}; adopting a compression method in the communication \cite{alistarh2018convergence}; and decentralizing the workers such that each worker only communicate with a subset of the others \cite{lian2017can,tang2018d} (details appear in Section~\ref{related work section}). 

However, in theory these algorithms are usually analyzed with protocol-specific assumptions. For instance, with asynchronous communication \cite{recht2011hogwild,alistarh2018convergence,de2015taming} the convergence is provably shown to depend on the maximum delay of two workers querying the shared model. 
By comparison, when workers communicate in
a decentralized fashion \cite{lian2017can,tang2018d,lian2017asynchronous}, the convergence rate depends on the spectral gap of the underlying graph \cite{lu2020moniqua}. This brings up several problems:
\begin{enumerate}
    \item The assumptions on which the convergence depend are highly correlated with the implementation details of the system. Consider a hybrid system where part of the system is communicating with shared memory and the others are communicating via message passing: extracting assumptions for such a system is complicated.
    \item It is hard to compare the convergence rate among protocols with different assumptions since they are usually measured in terms of different aspects of the system (e.g., the aforementioned maximum delay and spectral gap), and so the theory gives limited insights when comparing different protocols. 
    \item The analysis can be redundant. Even with an identical optimizer, a new communication protocol requires a new proof of convergence, expending valuable researcher time.
\end{enumerate}
At this point, a natural question is: \textit{Is there a generic way to analyze the underlying communication protocol such that convergence bounds can be obtained regardless of the system details?}

In this paper we answer this question in the affirmative by proposing MixML, a general framework that provides convergence rate guarantees for parallel algorithms independent of system details such as hardware hierarchy or configuration of networks. 
We do this by introducing and extending the concept of \emph{mixing time} from Markov Chain theory \cite{levin2017markov}. 
The motivation for using the mixing time originates from the insight that communication among multiple workers is a way of letting them reach consensus on certain values or parameters, and we can measure the effectiveness of a protocol by how rapidly consensus is (approximately) reached. For example, this consensus can be obtained via many atomic operations such as read-and-write \cite{de2017understanding}, averaging \cite{lian2017can}, etc. In other words, with sufficient communication or these atomic operations, workers will reach stationarity (consensus) after which more communication will not change their states (or at least will not change them very much).
The \emph{mixing time} for a given protocol essentially measures the time required by protocol to get ``close'' to such consensus state, in exactly the same way that the mixing time of a Markov chain measures how long it takes for the chain to get ``close'' to its stationary distribution. 

Our contributions in this paper can be summarized as follows.
\begin{itemize}[nosep]
    \item We propose MixML, a general framework that abstracts the system details of parallel learning and separates computation and communication. (Section~\ref{modeling section})
    \item We show how to derive convergence rates for parallel algorithms in the MixML framework. These rates depend on the communication protocol only in terms of its mixing time parameter $\tmix{}$. (Section~\ref{principled proof section})
    \item We apply our theory on a variety of algorithms, including SGD, AMSGrad and RMSProp and reveal how communication affects their convergence bound in general. We also show that our results match and sometimes improve several previous analysis. (Section~\ref{application section})
    \item We conduct experiments to empirically show the dependency of mixing time on convergence when training Resnet20 on CIFAR10. (Section~\ref{Experiments section})
\end{itemize}

%% file: section/Related_work.tex
\section{Related Work}\label{related work section}
\myparagraph{Parallel Learning with Weak Consistency.}
As introduced in Section~\ref{introduction section}, a variety of protocols work by maintaining weak consistency compare to the basic centralized synchronous protocol. For instance,
(1) Lock-free read-and-write \cite{recht2011hogwild,noel2014dogwild,lian2015asynchronous,pan2016cyclades}. This is usually adopted in a shared memory architecture, where multiple threads can read and write main memory without software locks. 
(2) Asynchronous communication \cite{agarwal2011distributed,lian2015asynchronous,li2014communication,de2015taming}. With asynchrony, workers update a shared or centralized model without synchronization. This is robust to stragglers (slow workers) in the system and can speed up training with respect to wall clock time. 
(3) Decentralized communication \cite{he2018cola,zhang2019asynchronous,nazari2019dadam,assran2018stochastic}. Decentralization removes centralized storage (i.e. Parameter Server \cite{li2014scaling} or shared memory \cite{recht2011hogwild}) and let parallel workers communicate in a peer-to-peer fashion. This achieves load-balancing and presents a easier way to scale up. 
(4) Sparse communication \cite{wang2018atomo,yu2019distributed}. In sparse communication, the information exchanged among workers is sparsified. For example, \cite{alistarh2018convergence} studies the case where only top $K$ largest coordinates of a gradient is sent in each communication step.
In other related work, \citet{yu2019distributed} discusses the convergence of momentum method in the decentralized setting and \citet{nazari2019dadam} proposes a modified version of Adam \cite{kingma2014adam} in the decentralized setting where projection is considered.

\myparagraph{Analysis on Weak Consistency.}
Unified analysis on optimization with weak consistency are often focused on Stochastic Gradient Descent (SGD) and its variants.
Specifically, \citet{lian2015asynchronous} provides the proofs for SGD with asynchronous communication where shared memory and message-passing setting are both considered.
\citet{de2015taming,alistarh2018convergence} adopt martingale-based analysis to prove the convergence of SGD with asynchronous communication in both convex and non-convex cases, which are more general than previous work \cite{ho2013more}.
\citet{wang2018cooperative} proposes a general model called Cooperative SGD that cover averaging and decentralization in parallel SGD.
\citet{alistarh2020elastic} proposes a elastic consistency model that thoroughly discusses the case of SGD with weak consistency.
\citet{qiao2018fault} proposes a concept of "rework cost" and analyze the staleness of asynchronous communication as perturbations of a series of SGD iterations. 
In other works, \citet{liu2018towards,hannah2019fundamental} studies the trade-off between momentum and asynchrony via specific streaming PCA problem. \citet{acharya2019distributed} theoretically discusses the sublinear communication in distributed training. \citet{karimireddy2019error} proposes a general framework $\delta$-compressor that specifically analyzes the weak consistency caused by quantization or sparsification scheme in the communication.


%% file: section/Communication.tex
\section{Modeling Parallel Learning}\label{modeling section}

\begin{figure}[t]
    \centering
    \includegraphics[width=0.6\textwidth]{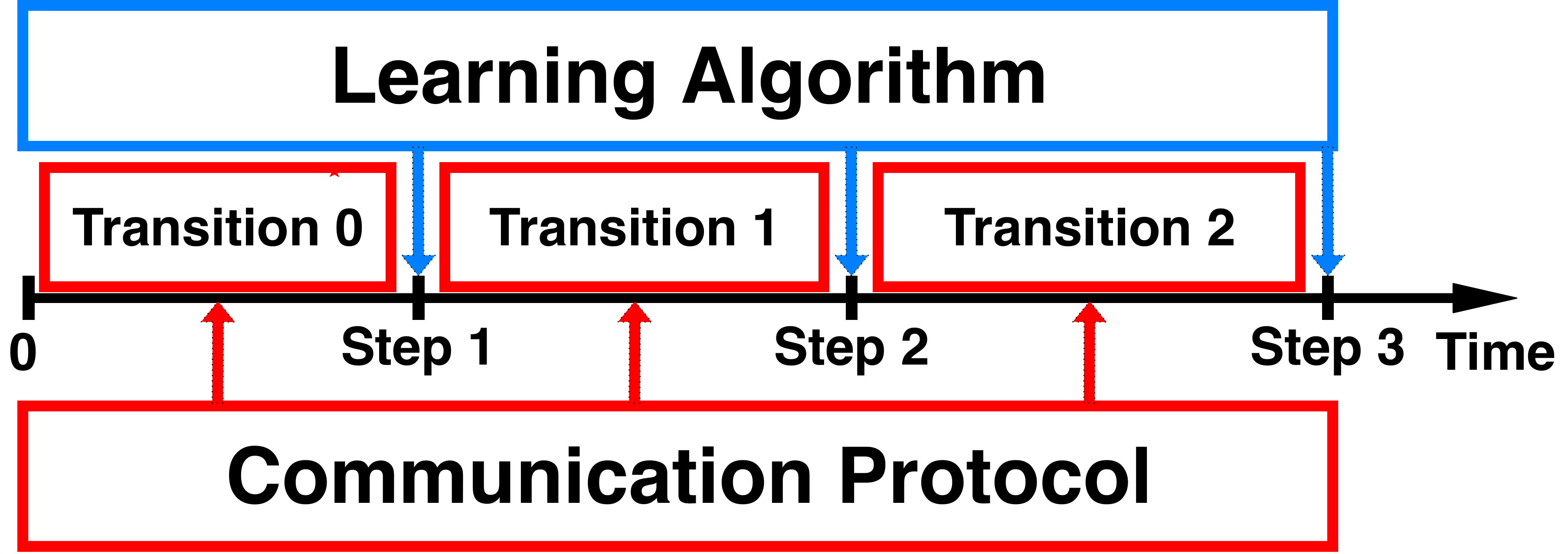}
    \caption{An illustration of how the computation and communication events conduct optimization steps and transitions over time.}
    \label{diagram}
\end{figure}

Consider multiple workers running a parallel learning algorithm, where workers can refer broadly to machines, hardware accelerators, threads, etc. Each worker has access to two types of storage resources: (1) internal storage (e.g.~local memory) which can only be accessed and modified by the worker itself and (2) external storage, which can potentially be read by or modified due to actions of multiple workers (e.g.~shared memory, caches backed by shared memory, communication buffers) \cite{de2017understanding}.
This storage is updated by two types of events: \emph{computation events}, in which a worker reads from and accumulates into one or more storage locations in order to implement the learning algorithm, and \emph{communication events}, in which either the workers actively or the network/cache hierarchy/interconnect passively alter data in the external storage to maintain consistency among the workers.
For example, in  multicore parallel SGD, a gradient computation by a CPU core would be a computation event, while an automatic cache update to reflect changes in RAM would be a communication event \cite{de2017understanding}.
Or, in a message-passing environment, messages would be communication events \cite{lian2015asynchronous},

In MixML, we make the following simplifying assumptions.
First, we suppose that the basic communication/computation events can be modeled as atomic. This is a fairly mild assumption, because the thing we are assuming are atomic is only a modification of a particular storage location such as a local cache: we do not require that information will propagate to all the workers atomically. 
Second, we suppose that the events follow a \emph{causal ordering}. Since no computation can depend on each other cyclically, we can then put their ordering into a \emph{total ordering} in which we assign each event a distinct time $t$ \cite{mullender1993distributed,schneider1993replication}.
Without loss of generality, we assign computation events to integer times $t$ starting at $t = 0$, and let communication events occur at non-integer times: this is illustrated in Figure~\ref{diagram}.
Note that in a real system, these events would be happening in parallel: we are requiring here that the system behave equivalently to the events happening sequentially in some order, not that the events actually are occurring sequentially in real time.
Third, since MixML models a learning algorithm, we suppose that the external storage state can be modeled as a vector of real numbers; we neglect any error due to using floating-point numbers rather than exact arithmetic in the algorithm.
We also suppose that each communication event acts \emph{linearly} on this state: this holds for essentially all communication events that occur in practice, including copying, accumulating, and averaging.

Proceeding from these assumptions, we model the system formally as follows.
Let $\*X_t$ denote the state of the external storage just before the event at time $t$. 
At time $t$, a worker then computes an update, which we denote $\*\Delta_t(\*X_t, \ldots, \*X_0)$ and accumulates it, multiplied by a step size $\alpha_t$, into the state.
(Note that $\*\Delta_t$ is a function of $\*X_s$ for $s \le t$ because the worker could have saved information from past computation actions into its internal storage, and so this past state could affect the present update.)
Then, before time $t+1$, some number of communication events might occur: each is a linear operator acting on $\*X$, and we let $\*M_t$ denote their product.
Thus, the dynamics of the state can be explicitly expressed as:
{\setlength{\abovedisplayskip}{3pt}\setlength{\belowdisplayskip}{3pt}\begin{equation}\label{Equation prototype}
    \*X_{t+1}\leftarrow \underbrace{\*M_t}_{\mathclap{\text{communication}}}(\*X_t+\alpha_t \underbrace{\*\Delta_t(\*X_t, \cdots, \*X_0)}_{\text{computation}}).
\end{equation}}%
Note that (\ref{Equation prototype}) is a general form that describes most parallel learning algorithms \cite{recht2011hogwild,lian2015asynchronous,pan2016cyclades,lian2017can}.
As the sequence of $\*M_t$ and $\*\Delta_t$ denote the progress or effects of communication and computation, respectively, in the rest of the section, we separate the two processes and investigate them individually.

\subsection{Characterizing Communication.}
In parallel training, the purpose of communication is to let workers approach or reach consensus about the model they are learning. That is to say, after sufficient communication actions, all the workers should have the same (or at least similar) ``view'' of the state of the algorithm.\footnote{From a systems perspective, this implies the \emph{liveness} of the protocol. This is a widely used assumption in traditional parallel computing research \cite{frommer1998asynchronous,el2005asynchronous}.}
The number of transitions required by the protocol to reach this consensus measures the effectiveness of the protocol, and we call this quantity the \emph{mixing time}.
The motivation and formulation of mixing time is inspired by the theory of Markov Chains. Specifically, the mixing time of a Markov Chain refers to the time or number of transitions required to get close to its stationary distribution~\cite{levin2017markov}. 

Now, we define our notion of mixing time, and introduce some assumptions we need to make for this notion to make sense.
Our assumptions are based on the following intuition about a general communication protocol:
\begin{enumerate}[nosep]
    \item The communication process should cause the system to approach a consensus state after a sufficiently long (but bounded) amount of time.
    \item The communication protocol should not amplify or enlarge any signal or value in the system. 
\end{enumerate}
Based on this intuition, we present the following assumptions about the communication process matrices $\*M_t$.
{\setlength{\abovedisplayskip}{2pt}\setlength{\belowdisplayskip}{2pt}
\begin{assumption}[Consensus is preserved by communication]\label{assume consensus}
    There exists a matrix $\*M_\infty$, such that for any time $t\geq 0$,
    \[ \*M_\infty \*M_t=\*M_t\*M_\infty=\*M_\infty=\*M_\infty^2. \]
\end{assumption}
\begin{assumption}[Mixing time]\label{assume mixing time}
There exists a constant $\tmix{}$, such that for any state vector $\*X$ and for any $t\geq 0$, the following bound holds: \[\textstyle\left\|\left( \prod_{k=t}^{t+\tmix{}-1}\*M_k \right) \*X-\*M_\infty \*X\right\| \leq \frac{1}{2}\left\|\*X-\*M_\infty \*X\right\|.\]
\end{assumption}
These assumptions abstract our first piece of intuition. Assumption~\ref{assume consensus} shows starting from state $\*X$ both that a consensus state $\*M_{\infty} \*X$ exists, that additional communication actions will not alter this state, since $\*M_t \*M_{\infty} \*X = \*M_{\infty} \*M_t \*X = \*M_{\infty} \*X$.
Assumption~\ref{assume mixing time} ensures that communication events will cause the state to approach this consensus.
\begin{assumption}\label{assume scale bound}
There exists a constant $\xi$ such that for any $t \geq s \geq 0$, $\left\|\prod_{k=s}^{t} \*M_k\right\| \leq \xi$.
\end{assumption}
This assumption abstracts our second piece of intuition, by ensuring that communication never increases the magnitude of a value it communicates.
Given Assumptions~\ref{assume consensus} to \ref{assume scale bound}, we can bound the rate of reaching consensus as follows.
\begin{lemma}\label{lemmamixingtime}
For any $t \geq s \geq 0$, and any state vector $\*X$,
\begin{equation}
    \textstyle
    \left\| \left(\prod_{k=s}^{t-1} \*M_k \right) \*X-\*M_\infty \*X\right\| \leq 2^{-\left\lfloor\frac{t-s}{\tmix{}}\right\rfloor}(1+\xi)\xi\left\|\*X\right\|.
\end{equation}
\end{lemma}}
Lemma~\ref{lemmamixingtime} shows that, without any computation events, the system approaches consensus at a rate that depends only on its mixing time.

\myparagraph{Communication case studies.}
We proceed to introduce how mixing time can be applied to analyze the communication protocols of several representative algorithms and how mixing time relates to assumptions from previous work.

\textit{I. AllReduce} \cite{chen2016revisiting}.
AllReduce is an operation from MPI \cite{gropp1999using} that lets workers reach a consensus state immediately. Specifically, all the workers in AllReduce broadcast their local values to all the others while gathering remote values from all the other workers. In our formulation, it is straightforward to show that mixing time for AllReduce is\footnote{Note that in our setting, the optimization timestep is defined as one single worker's update anywhere in the system. In a synchronous centralized algorithm, one iteration implies $n$ such steps, one for each worker. And correspondingly, the mixing time is $n$.} $n$, where $n$ is the number of workers.

\textit{II. Gossip-based Averaging} \cite{boyd2004fastest}. 
Gossip-Based Averaging is a decentralized protocol that let all the parallel workers connect to form a graph. And each worker communicates by averaging its parameters with adjacent neighbors according to a doubly-stochastic matrix $W$.
\citet{boyd2004fastest} shows that this $W$ can be seen as a ``transition'' of the system and the rate at which it converges depends on the Markov Chain $W$ defines. Note that this protocol can be seen as a special case of our formulation as it has special transition matrix (Markov Chain).
Applications of using this protocol include D-PSGD \cite{lian2017can}, which parallelizes SGD with gossip-based averaging.

\textit{III. Lock-free Method} \cite{recht2011hogwild}. With the lock-free method, parallel workers communicate via a shared memory but without software locks. An application is Hogwild! \cite{recht2011hogwild}, where workers communicate via read operations to shared memory. Hogwild! assumes that the maximum delay between when a gradient update is written by a worker and when it is available in all workers' local caches is bounded by some constant $\tau$. This means that, after at most $\tau$ time, all the workers will reach perfect consistency (making $\tau$ an upper bound on our mixing time $\tmix{}$). In comparison, our mixing time formulation is less restricted as it only requires the time of distance to consensus being halved is bounded.

%% file: section/MixML.tex
\subsection{Characterizing Computation}\label{MixML section}
To characterize computation, we want to formalize the intuition that the parallel algorithm we are modeling should be in some sense a ``parallel version'' of a known sequential stochastic learning algorithm $\mathcal{A}$.
Such sequential learning algorithms have the general form
\begin{equation}\label{generalrulesequential}
    \*x_{t+1} = \*x_t - \alpha_t \*{\tilde \delta}_t(\*x_t, \cdots, \*x_0).
\end{equation}
where $\*x_t \in \mathbb{R}^d$ is the state at timestep $t$, $\alpha_t$ denotes the step size, and
$\*{\tilde \delta}_t(\*x_t, \cdots, \*x_0)$ denotes some randomized update to variable $\*x_t$ (usually a gradient step) that could depend on previous states.
Most sequential learning algorithms used in practice are described by (\ref{generalrulesequential}), including SGD, Momentum SGD, and Adam \cite{ghadimi2013stochastic,johnson2013accelerating,kingma2014adam,reddi2019convergence}.

What would it mean for a learning algorithm described in our MixML model to be a ``parallel version'' of Algorithm $\mathcal{A}$?
We start by considering a special case, where we suppose that workers always reach consensus immediately: that is, $\*M_t = \*M_\infty$. 
We call this \emph{perfect communication}. Under perfect communication, we say that our algorithm is a ``parallel version'' of $\mathcal{A}$ if the execution of the parallel algorithm follows the same dynamics as its sequential execution.
And the perfect-communication parallel algorithm has update step from (\ref{Equation prototype})
\begin{equation}
    \*X_{t+1} = \*M_\infty (\*X_t+\alpha_t \*\Delta_t(\*X_t, \cdots, \*X_0)),
\end{equation}
so how can we meaningfully say this is equivalent to (\ref{generalrulesequential})?
We do this by observing that the worker that computes at timestep $t$ does not usually read the whole state $\*X_t$, but rather some \emph{local view} of the external storage it has access to, which we denote $\*\Pi_t \*X_t$, where $\*\Pi_t$ is a coordinate projection operator.\footnote{A coordinate projection operator here means a matrix that selects out some coordinates of its input: i.e. a full-rank matrix with entries in $\{0,1\}$ and at most one $1$ in each row and column.}
Two such projections are important. The first is the model used for computation, determined by a coordinate projection matrix\footnote{As an example, in a shared memory architecture, each worker thread fetches the model from RAM into its own cache and computes an update based on the cached parameters. The parameters in the cache is the local view for the worker, which could differ from what is stored in RAM or in other workers' caches.} $\*{\hat \Pi}_t$.
The second is a coordinate projection $\*\Pi_*$ that would ``select out'' some components of a state $\*X_t$ of the parallel algorithm that would be representative of the state $\*x_t$ of the sequential algorithm.\footnote{For example, in a shared memory setup, $\*\Pi_* \*X_t$ could represent the model stored centrally in RAM.}
We formalize this equivalence assumption as follows.
{\setlength{\abovedisplayskip}{2pt}\setlength{\belowdisplayskip}{2pt}\begin{assumption}\label{assume projection}(Parallel version)
There exists coordinate projection matrices $\*\Pi_*$, $\*\Pi_t$, and $\*{\hat \Pi}_t$ such that for all $t \ge 0$, and any state vectors $\*Z_0, \ldots, \*Z_t$,
\[
    \*\Delta_t (\*Z_t, \ldots, \*Z_0)
    =
    \*{\hat \Pi}_t \*{\tilde \delta}_t(\*\Pi_t \*Z_t, \ldots, \*\Pi_0 \*Z_0),
\]
and it always holds that
\[\*\Pi_t \*M_\infty = \*\Pi_* \*M_{\infty}
\hspace{2em}\text{and}\hspace{2em}
\*M_{\infty} \*{\hat \Pi}_t = \*\Pi_* \*M_{\infty}.\]
\end{assumption}}
This assumption formalizes the idea that what a computation step does is (1) read some subset of the external storage, (2) treat it as if it were the state of the sequential algorithm and compute the sequential update, and (3) accumulate that update somewhere into external storage.
The latter part of this assumption encodes the idea that, if the algorithm has reached consensus, then any worker's ``view'' of the state should also be representative of the state of the sequential algorithm. 
With this assumption, if we set $\*x_t = \*\Pi_* \*M_{\infty} \*X_t$, which we call the \emph{consensus state} of the algorithm at time $t$, then the update step of the parallel algorithm under perfect communication is exactly in the form of (\ref{generalrulesequential}).
This formalizes our notion of a computation process being a ``parallel version'' of a sequential algorithm.

Now, consider the setting where there is no perfect communication and so workers fail to reach consensus state before stepping into the next iteration, resulting in a state which we refer to as being \emph{weakly consistent}.
In such setting, the computation events will depend on some intermediate version of the weights---a different sequence from the consensus state $\*x_t$---and so the parallel algorithm will not be exactly described by (\ref{generalrulesequential}).
To address this, we introduce a new weight sequence defined by $\*u_t = \*\Pi_t \*X_t$ which represents the \emph{local view} of the worker that updates at time $t$, as described above.
Then, the update step of the parallel, weakly consistent algorithm can be written as
\begin{equation}\label{generalrulesequential}
    \*x_{t+1} = \*x_t - \alpha_t \*{\tilde \delta}_t(\*u_t, \cdots, \*u_0).
\end{equation}

Note that due to weak consistency, $\{\*u_k\}_{k\leq t}$ is generally different from $\{\*x_k\}_{k\leq t}$.
For example, the two sequences can differ due to staleness \cite{lian2015asynchronous,agarwal2011distributed} or sparse connection \cite{lian2017can}.


%% file: section/principle.tex
\section{A Principled Proof Approach}\label{principled proof section}

In this section we show how to prove convergence rates with MixML by ``adapting'' a convergence proof for a sequential algorithm into one for a parallel version.
We start by noting that our assumed causal ordering is also an ordering of the randomness in the learning algorithm, so we can define $\mathcal{F}=\{\mathcal{F}_t\}_{t\geq 0}$ which captures the algorithmic randomness that has happened up to time $t$.
Note that we are assuming the communication process is deterministic, so all this randomness is coming from the stochastic updates $\*{\tilde \delta}_t(\*u_t, \ldots, \*u_0)$.
If we define $\*\delta_t = \mathbb{E}[\*{\tilde \delta}_t]$ to be the \emph{expected} sequential-algorithm update, and for brevity define
{\setlength{\abovedisplayskip}{2pt}\setlength{\belowdisplayskip}{2pt}\begin{align*}
    \*\delta_t^{(x)} &= \*\delta_t(\*x_t, \cdots, \*x_0)
    &
    \*{\tilde \delta}_t^{(x)} &= \*{\tilde \delta}_t(\*x_t, \cdots, \*x_0) \\
    \*\delta_t^{(u)} &= \*\delta_t(\*u_t, \cdots, \*u_0)
    &
    \*{\tilde \delta}_t^{(u)} &= \*{\tilde \delta}_t(\*u_t, \cdots, \*u_0),
\end{align*}}
then the update rule of a parallel algorithm in a stochastic setting under the MixML framework can be written as:
\begin{equation}\label{Equation parallel update}
\begin{aligned}
    \*x_{t+1} &= \*x_t - \alpha_t\*{\tilde{\delta}}_t^{(u)} \\ \nonumber
    &= \underbrace{\*x_t - \alpha_t\*\delta_t^{(x)}}_{\mathclap{\text{Sequential Update}}} + \underbrace{\alpha_t(\*\delta_t^{(u)}-\*{\tilde{\delta}}_t^{(u)})}_{\mathclap{\text{Sampling Noise}}} + \underbrace{\alpha_t(\*\delta_t^{(x)}-\*\delta_t^{(u)})}_{\mathclap{\text{Weak Consistency Noise}}}.
\end{aligned}
\end{equation}
Note that the update in (\ref{Equation parallel update}) has three parts: the one that is associated with non-stochastic sequential update, sampling noise, and the noise from weak consistency. In other words, if we can obtain a bound on the weak consistency noise, then we can just fit this bound into the corresponding sequential proof (which was already able to ``handle'' the first two terms) and obtain the convergence bound for its parallel version.
To do this, we need the following assumptions.
\begin{assumption}\label{assumesampler}
The sampling in the update is unbiased and has a bounded variance, specifically, for $\forall t\geq 0$,
\begin{align*}
\mathbb{E}\left[\*{\tilde{\delta}}_t^{(u)}\Big|\mathcal{F}_t\right]=\*\delta_t^{(u)}, \hspace{1em}\mathbb{E}\left[\|\*{\tilde{\delta}}_t^{(u)}-\*\delta_t^{(u)}\|^2\Big|\mathcal{F}_t\right] \leq \sigma^2
\end{align*}
\end{assumption}
\begin{assumption}\label{assumecontinous}
The update of the algorithm is Lipschitz continuous: there exists $\mathcal{L}$ such that for any $t\geq 0$, there exists $L_{k,t}$ that for any $\*y_t, \cdots, \*y_0$ and $\*z_t, \cdots, \*z_0\in\mathbb{R}^d$, 
\begin{align*}
    \textstyle
    \|\*\delta_t(\*y_t, \cdots, \*y_0)-\*\delta_t(\*z_t, \cdots, \*z_0)\| \leq \sum_{k=0}^{t}L_{k,t}\|\*y_k-\*z_k\|
\end{align*}
and $\sum_{k=0}^{t}L_{k,t}\leq\mathcal{L}$.
\end{assumption}
Based on the assumptions above, running an arbitrary learning algorithm in the form of (\ref{Equation parallel update}) under the MixML framework gives the bound shown in the following lemma.
\begin{lemma}\label{lemma general bound}
If $\{\alpha_t\}_{t\geq 0}$ is non-increasing, then
\begin{equation}
    \sum_{t=0}^{T-1}\alpha_t\mathbb{E}\|\*\delta_t^{(x)}-\*\delta_t^{(u)}\|^2 \leq 16(1+\xi)^2\xi^2\tmix^2{}\mathcal{L}^2\sum_{t=0}^{T-1}\alpha_t^3\mathbb{E}\|\*\delta_t^{(x)}\|^2 + 4(1+\xi)^2\xi^2\tmix{}\sigma^2\mathcal{L}^2\sum_{t=0}^{T}\alpha_t^3
\end{equation}
\end{lemma}
We make a few observations on the bound in Lemma~\ref{lemma general bound}. The LHS measures the accumulated ``extra'' noise from parallelism as compared to a sequential algorithm.
Note that in the RHS, the bound depends on the update in the sequential algorithm and the sampling noise. Note that this bound is independent of the explicit expression of the sequential updates $\*\delta_t^{(x)}$ and does not contain any terms regarding parallel updates $\*{\tilde{\delta}}_t^{(u)}$ or the parallel views $\*u_t$. In other words, we can easily use this bound together with a proof for a sequential algorithm to obtain the convergence rate for its parallel version, as its proof will take two simple steps: (1) find the accumulated error terms as shown in the LHS of Lemma~\ref{lemma general bound}; and (2) fold in the parameters $\mathcal{L}$ and $\*\delta_t^{(x)}$ directly. 


We can compare this bound to the previous work of \citet{alistarh2020elastic} on elastic consistency, where the co-relation of $\*x_t$ and $\*u_t$ is assumed specifically on SGD and has the form of:
$\mathbb{E}[\|\*x_t-\*u_t\||\mathcal{F}_t]\leq \alpha_t B$,
where $B$ is a constant.
Lemma~\ref{lemma general bound} provides more detailed insight about how updates from local view and global view can be differed, and how they are affected by the parameters of the algorithms.

%% file: section/Application.tex
\section{Applications}\label{application section}
In this section, we apply MixML to a variety of algorithms and show how it helps obtain tight convergence bounds that match case-specific analyses from previous works.
\subsection{Stochastic Gradient Descent.}
We start from Stochastic Gradient Descent (SGD), a widely adopted algorithm for large-scale machine learning \cite{zhang2004solving,bottou2010large}. 
SGD works by iteratively taking a step in the reverse direction of the stochastic gradient at the current point. 
Its update rule fits in MixML with $\*{\tilde{\delta}}_t^{(u)}=\nabla\tilde{f}(\*u_t)$, where $\nabla\tilde{f}(\*u_t)$ denotes the stochastic gradient at time t and $\mathbb{E}[\nabla\tilde{f}(\*u_t)]=\nabla f(\*u_t)$. The update rule of SGD can be written as:
\begin{equation}
    \*x_{t+1}=\*x_t-\alpha_t\nabla\tilde{f}(\*u_t).
\end{equation}
Consider running SGD on a smooth non-convex function with $L$-Lipchitz gradients, i.e. for any $\*x, \*y\in\mathbb{R}^d$, $\| \nabla f(\*x)-\nabla f(\*y)\| \leq L\|\*x-\*y\|$.
Then the following lemma holds.
\begin{lemma}\label{lemma SGD nonconvex}
If we run SGD on a smooth non-convex function for $T$ iterations, then it converges at the following rate:
\begin{equation}
    \sum_{t=0}^{T-1}\alpha_t\mathbb{E}\left\|\nabla f(\*x_t)\right\|^2 \leq 2(f(\*0)-f(\*x_T)) + \sigma^2L\sum_{t=0}^{T-1}\alpha_t^2 + \sum_{t=0}^{T-1}\alpha_t\mathbb{E}\|\nabla f(\*x_t)-\nabla f(\*u_t)\|^2.
\end{equation}
\end{lemma}
We can see in Lemma~\ref{lemma SGD nonconvex} that the convergence bound depends on the accumulated noise from parallelism as we formulated before. With perfect communication, $\*u_t=\*x_t, \forall t\geq 0$, then we recover the convergence rate for sequential SGD \cite{alistarh2020elastic}.
Note that update for SGD only depends on the current iteration, it is straightforward to verify for SGD $\mathcal{L}=L$.
By fitting in Lemma~\ref{lemma general bound} with $\mathcal{L}=L$ and $\*\delta_t^{(x)}=\nabla f(\*x_t)$, we are able to provide bound on parallel SGD as shown in the following Theorem.
\begin{theorem}\label{thm SGD}
If we run SGD on a smooth non-convex function for $T$ iterations under the framework of MixML, then SGD converges at the following rate:
\begin{equation}
    \sum_{t=0}^{T-1}\alpha_t\left(1-16(1+\xi)^2\xi^2\tmix^2{}L^2\alpha_t^2\right)\mathbb{E}\left\|\nabla f(\*x_t)\right\|^2 \leq 2(f(\*0)-f(\*x_T)) + \sigma^2L\sum_{t=0}^{T-1}\alpha_t^2 + 4(1+\xi)^2\xi^2\tmix{}\sigma^2L^2\sum_{t=0}^{T}\alpha_t^3.
\end{equation}
\end{theorem}
With specifically assigned step size, we are also able to show the convergence rate depending on $T$ in the following Corollary:
\begin{corollary}\label{corollary SGD non-convex}
If we assign $\alpha_t=\frac{\sqrt{2(f(\*0)-f(\*x_T))}}{\sigma\sqrt{LT}+\tmix{}}$, and run SGD for sufficiently large number of iterations such that $T\geq \frac{64(f(\*0)-f(\*x_T))\xi^2(1+\xi)^2\tmix^2{}L}{\sigma^2}$, then SGD converges at the following rate:
\begin{equation}
    \frac{1}{T} \sum_{t=0}^{T-1}\mathbb{E}\left\|\nabla f(\*x_t)\right\|^2 \leq \frac{4\sigma\sqrt{2(f(\*0)-f(\*x_T)L}}{\sqrt{T}} + \frac{2\sqrt{2(f(\*0)-f(\*x_T))}\tmix{}}{T} + \frac{16(1+\xi)^2\xi^2\tmix{}(f(\*0)-f(\*x_T))L}{T}.
\end{equation}
\end{corollary}
Corollary~\ref{corollary SGD non-convex} reveals two important insights: First, the communication does not affect the leading term in the convergence (the sampling complexity in the $1/\sqrt{T}$ term) and thus the optimization still obtains the tight $O\left(\frac{(f(\*0)-f(\*x_T))L\sigma^2}{\epsilon^4}\right)$ bound to obtain a stationary point $\*x$ such that $\|\nabla f(\*x)\|\leq \epsilon$ \cite{carmon2019lower,arjevani2019lower}.
Second,
it shows a unified convergence rate for SGD at a rate of $O(1/\sqrt{T}+\tmix{}/T)$. 
With perfect communication, i.e. $\tmix{}=1$, the rate matches result for sequential SGD.
By assigning different $\tmix{}$, we are able to obtain convergence bound for SGD on different protocols. We next provide detailed discussion in three cases. 
\paragraph{Centralized Asynchronous SGD.}
We start from the centralized asynchronous setting \cite{lian2015asynchronous} where both message-passing and shared memory are considered. Previous works \citet{lian2015asynchronous,de2015taming,recht2011hogwild} show that SGD with asynchronous communication converges at a rate of $O(1/\sqrt{T}+\tau/T)$ where $\tau$ refers to the maximum delay among workers conducting consecutive optimization steps. Recall in the case studies from Section~\ref{modeling section} that $\tau=\tmix{}$. As a result, Corollary~\ref{corollary SGD non-convex} provides a bound that matches with previous result.

\paragraph{Sparsified SGD.}
We continue to discuss the sparsified SGD \cite{wang2018atomo}. For simplicity we consider the general case where only a subset of model parameters can be successfully sent and received in the communication \cite{yu2018distributed}. In other words, the communication frequency of each coordinate of local models is reduced by a factor of $\eta$ due to sparsification. That is to say, for each coordinate, the mixing time is enlarged by a factor of $1/\eta$. As a result, the convergence rate for sparsified SGD then can be directly obtained by fitting $\tmix{}=1/\eta$ to Corollary~\ref{corollary SGD non-convex} as $O(1/\sqrt{T}+1/\eta T)$ (This corresponds to the Theorem 1 in \cite{yu2018distributed}).

\paragraph{Decentralized SGD.}
We proceed to discuss decentralized SGD \cite{lian2017can,lu2020moniqua}. 
Comparing Theorem 1 in \cite{lian2017can} and Corollary~\ref{corollary SGD non-convex}, MixML improves the rate on communication complexity term while both rates match on the sampling complexity term $1/\sqrt{T}$. Specifically, for communication term, \citet{lian2017can} shows $n/\lambda^2T$ where $\lambda$ refers to the spectral gap of the communication matrix. In contrast, MixML shows a dependency of $\tmix{}/T$. As shown in Markov Chain theory \cite{levin2017markov} and case studies in Section~\ref{modeling section}, here $\tmix{}$ is upper bounded by $\log(n)/\lambda$ when $\lambda$ is not zero and $1$ otherwise. That implies $\tmix{} < \log(n)/\lambda < n/\lambda^2$, and thus MixML obtains a tighter bound. 

\subsection{Stochastic Adaptive Methods.}
Aside from SGD, there are a variety of learning algorithms being widely adopted in learning systems. As SGD has been well studied in the literature regarding its parallel versions, the other algorithms are not fully understood in their parallel executions. Here we fill this gap and show how their convergence bounds in the parallel setting can be obtained under MixML.
We proceed to discuss other algorithms beyond SGD. We consider a general udpate rule in the form of the following:
\begin{equation}\label{updateruleAdam}
    \*x_{t+1} = \*x_t - \alpha_t\*V_t^{-p}\*m_t.
\end{equation}
where
\begin{align*}
    \*m_t &= \beta_1\*m_{t-1} + (1-\beta_1)\*g_t\\
    \*v_t &= \max\{\beta_2\*v_{t-1} + (1-\beta_2)\*g_t^2, \*v_{t-1}\}\\
    \*V_t &= \operatorname{diag}(\*v_t), \hspace{1em} \*g_t = \nabla\tilde{f}(\*u_t).
\end{align*}
where $\nabla\tilde{f}(\*u_t)$ denotes the stochastic gradient computed at $t$-th iteration. Note that Equation~\ref{updateruleAdam} is a general form that can cover a varirty of algorithms. For example, Momentum SGD \cite{yang2016unified} ($p=0$), AMSGrad \cite{reddi2019convergence} ($p=1/2$), RMSProp\footnote{We consider the corrected version of RMSProp, as the original version was shown diverging on several simple problems in \cite{reddi2019convergence}.} \cite{reddi2019convergence} ($p=1/2$, $\beta_1=0$), PAdam \cite{chen2018convergence}, etc. We generally refer this setting as \emph{Stochastic Adaptive Methods (SAM).}
Considering the update of SAM depends on all its previous states, it generally fits the setting of MixML. Specifically, we can fit SAM in MixML by
letting $\*{\tilde{\delta}}_t^{(u)}=\*V_t^{-p}\*m_t$ in Equation~\ref{Equation parallel update}.
And $\mathcal{L}$ for SAM can be obtained via the following Lemma.
\begin{lemma}\label{lemma L SAM}
Let $c$ denote a positive numerical constant\footnote{This constant prevents the denominator from being zero, which is commonly adopted assumption in previous works \cite{chen2018convergence}.} such that for any $1\leq j\leq d$ it holds $|\*e_j^\top \*v_0|\geq c$, and for any $\*x\in\mathbb{R}^d$ it holds that $\|\nabla\tilde{f}(\*x)\|_\infty\leq G_\infty$, then for SAM:
\begin{equation}
    \mathcal{L} = \frac{(c+2pG_\infty)L^2}{c^{2p+1}}\max\left\{2, \frac{4pG_\infty}{c}\right\}
\end{equation}
\end{lemma}
We obtain the convergence bound for SAM as shown in the follow Theorem.
\begin{theorem}\label{thm SAM}\label{}
If we run SAM on a smooth non-convex function with Liptchiz constant being L for T iterations under the framework of MixML, then SAM converges at the rate below
\begin{equation}
    \sum_{t=0}^{T-1}\alpha_t\mathbb{E}\left\|\nabla f(\*x_t)\right\|^2
    \leq C_1 + C_2\sum_{t=0}^{T-1}\alpha_t^2 + C_3\sigma^2\sum_{t=0}^{T-1}\alpha_t^2 + C_4\tmix{}\sigma^2\sum_{t=0}^{T}\alpha_t^3.
\end{equation}

where
\begin{align*}
    C_1 = & 2G_\infty^{2p}\left(\mathbb{E}f\left(\frac{\*x_0}{1-\beta_1}\right) - \mathbb{E}f^*\right)\\
    C_2 = & 2\left(3L + \frac{6L\beta_1^2}{1-\beta_1}\right)\frac{G_\infty^{2-2p}d}{(1-\beta_2)^{2p}\left(1-\frac{\beta_1}{\beta_2^{2p}}\right)}\\\
    C_3 = & \frac{4LG_\infty^{2p}}{(1-\beta_1)^2}\\
    C_4 = & 8(1+\xi)^2\xi^2\frac{(c+2pG_\infty)^2L^4G_\infty^{2p}}{c^{4p+2}}\left(\frac{2(1+\beta_1)(2-\beta_1)G_\infty^{2p}}{(1-\beta_1)^2} + \frac{2-\beta_1}{1-\beta_1}\right)\max\left\{4, \frac{16p^2G_\infty^2}{c^2}\right\}
\end{align*}
\end{theorem}
Theorem~\ref{thm SAM} provides a more general bound compared to Theorem~\ref{thm SGD} and it reveals the dependency of mixing time on a general learning algorithm under MixML. 
If we assign $\tmix{}=1$, then we recover the convergence rate of, for example, the Momentum SGD \cite{yang2016unified}, PAdam \cite{zhou2018convergence}, RMSProp \cite{zhou2018convergence}, etc.
And we also improves several results, e.g. parallel Momentum SGD \cite{yu2019distributed}.

In general, MixML reveals for parallel learning, the communication complexity is non-dominant in the final rate as $\tmix{}$ is associated with a fast decreasing  term ($O(\alpha^3)$). This actually shows the parallelism does not increase the total complexity compared to the sample complexity of a sequential method on smooth non-convex problem \cite{carmon2019lower,arjevani2019lower}.

%% file: section/Experiments.tex
\begin{figure*}[t]
    \subfigure[Training Loss over Epochs\newline Reduced Frequency Method]{
        \includegraphics[width=0.23\textwidth]{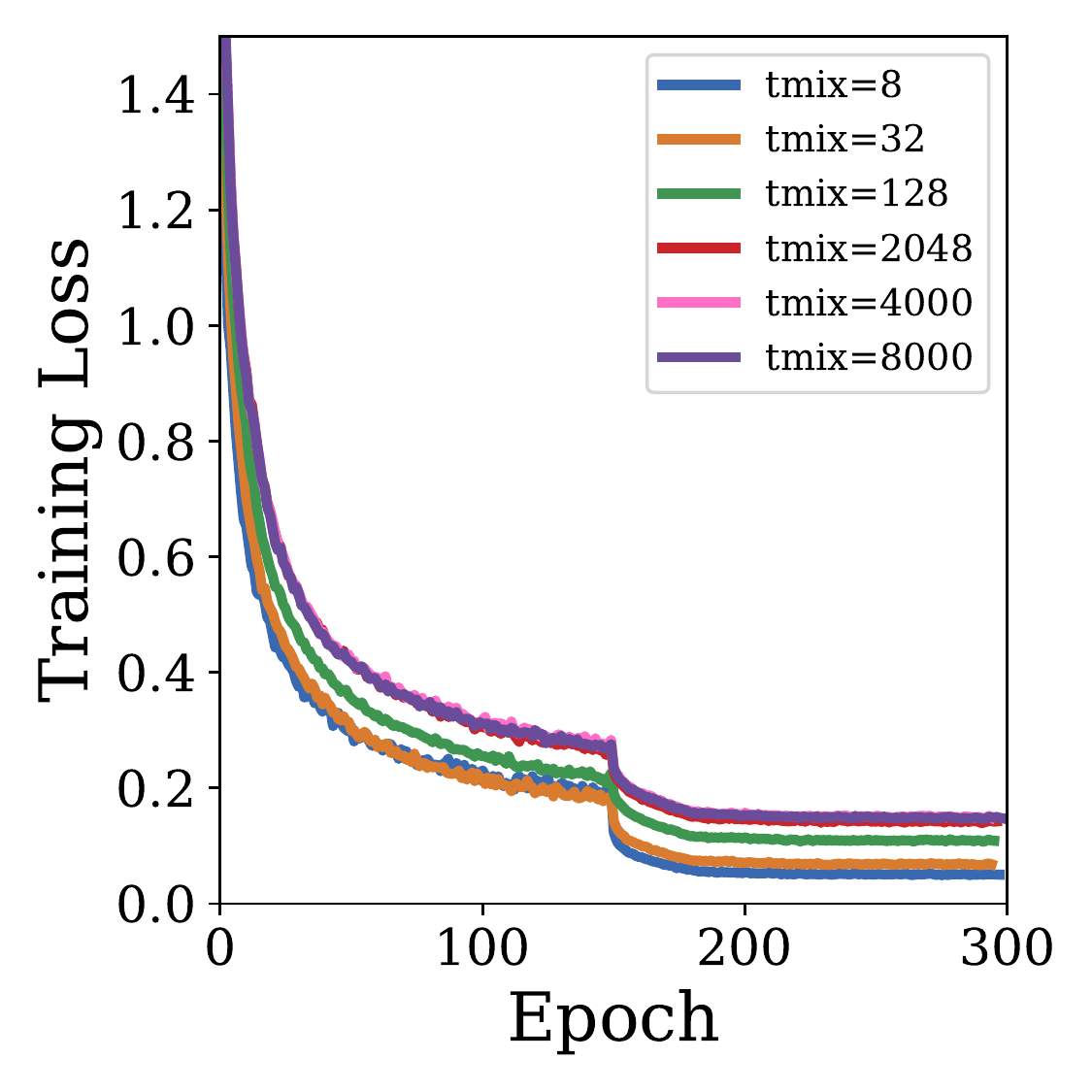}
        \label{train_loss_fre}
    }
     \subfigure[Training Loss over Epochs\newline Slack Matrix Method]{
        \includegraphics[width=0.23\textwidth]{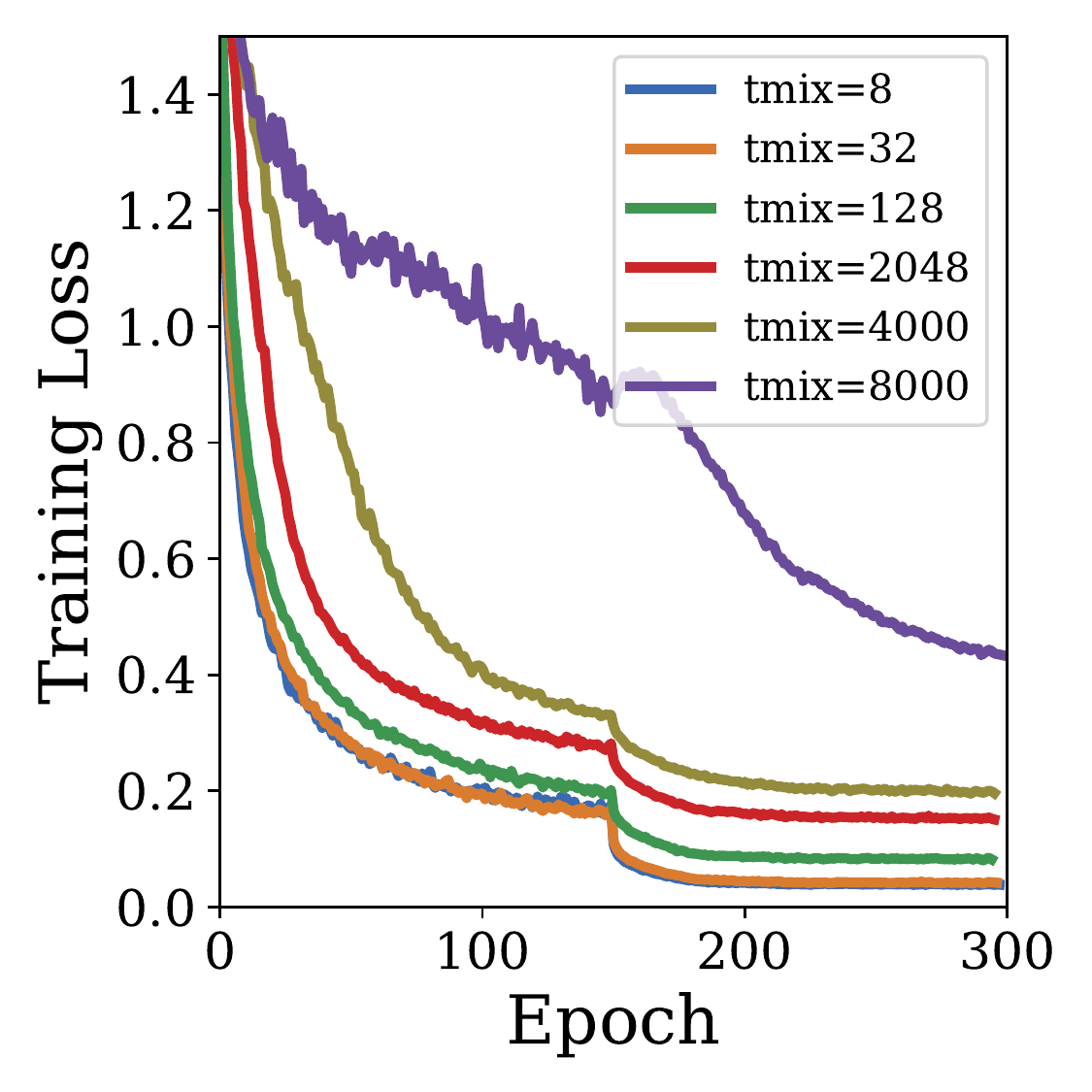}
        \label{train_loss_gamma}
    }
    \subfigure[
    Stationary Distance \newline Reduced Frequency Method
    ]{
        \includegraphics[width=0.23\textwidth]{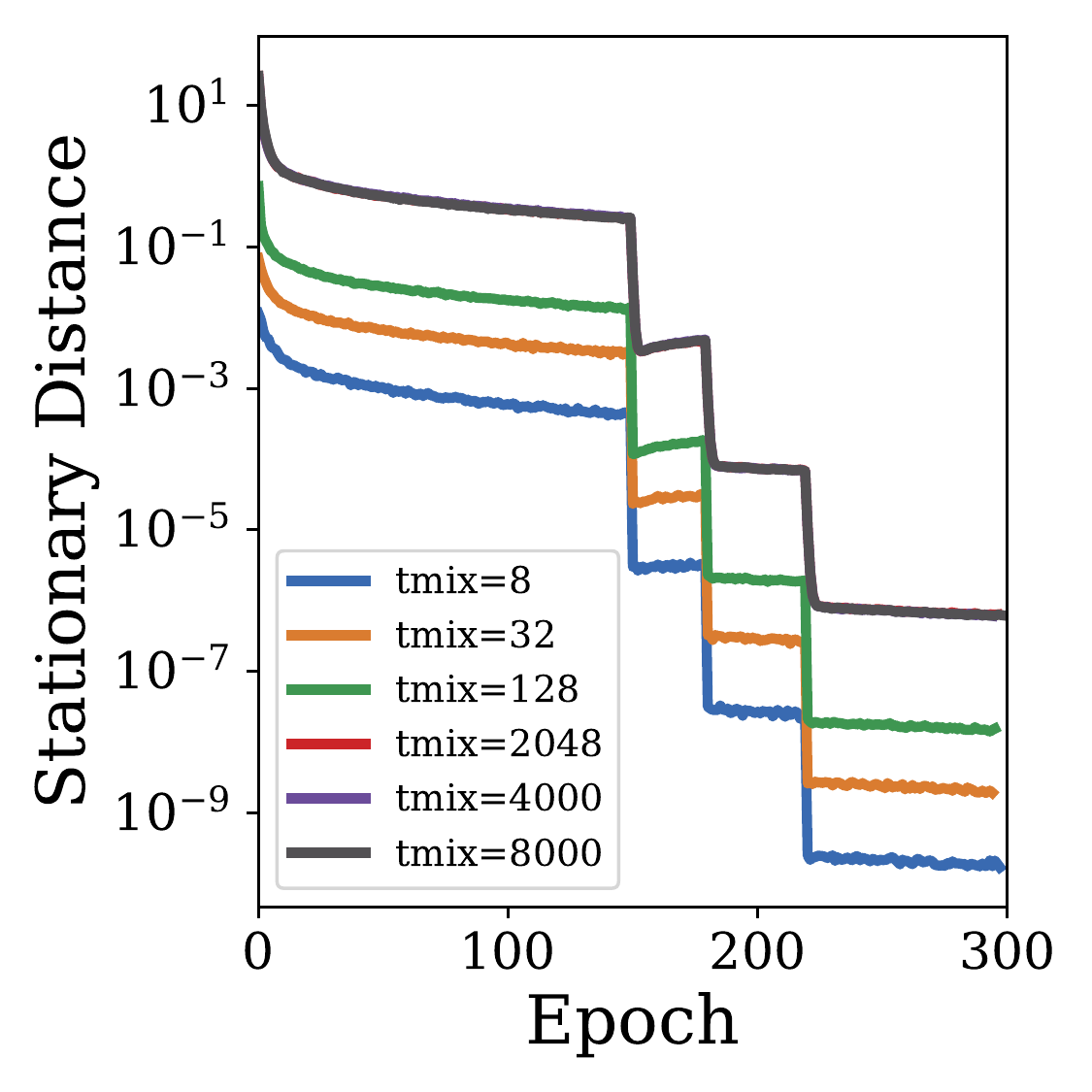}
        \label{consensus_fre}
    }
    \subfigure[Stationary Distance \newline Slack Matrix Method]{
        \includegraphics[width=0.23\textwidth]{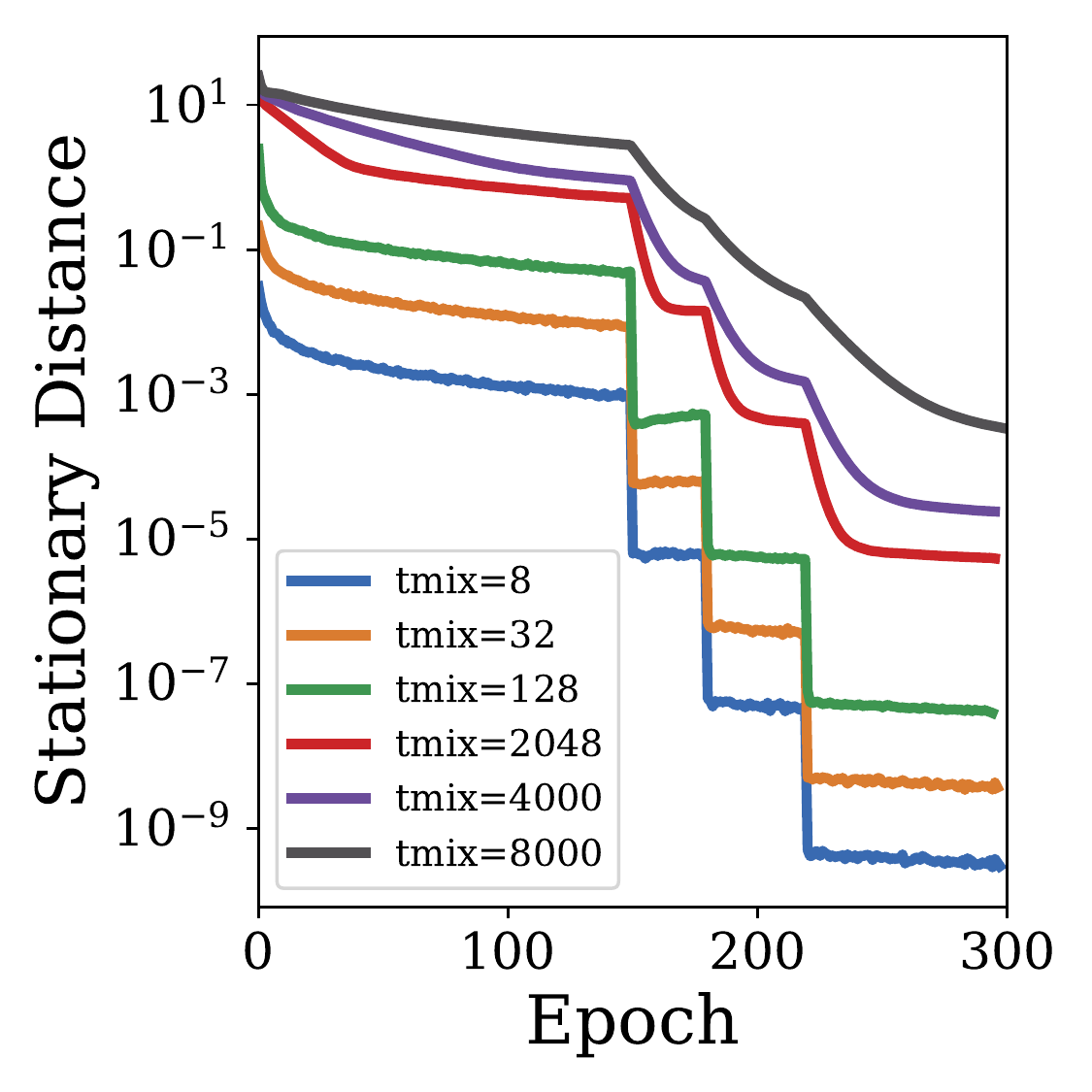}
        \label{consensus_gamma}
    }
    \caption{Measurement of training loss and stationary distance over epochs with different mixing time under reduced frequency method and slack matrix method.}
    \label{performance}
\end{figure*}
\begin{figure}
    \centering
    \includegraphics[width=0.55\textwidth]{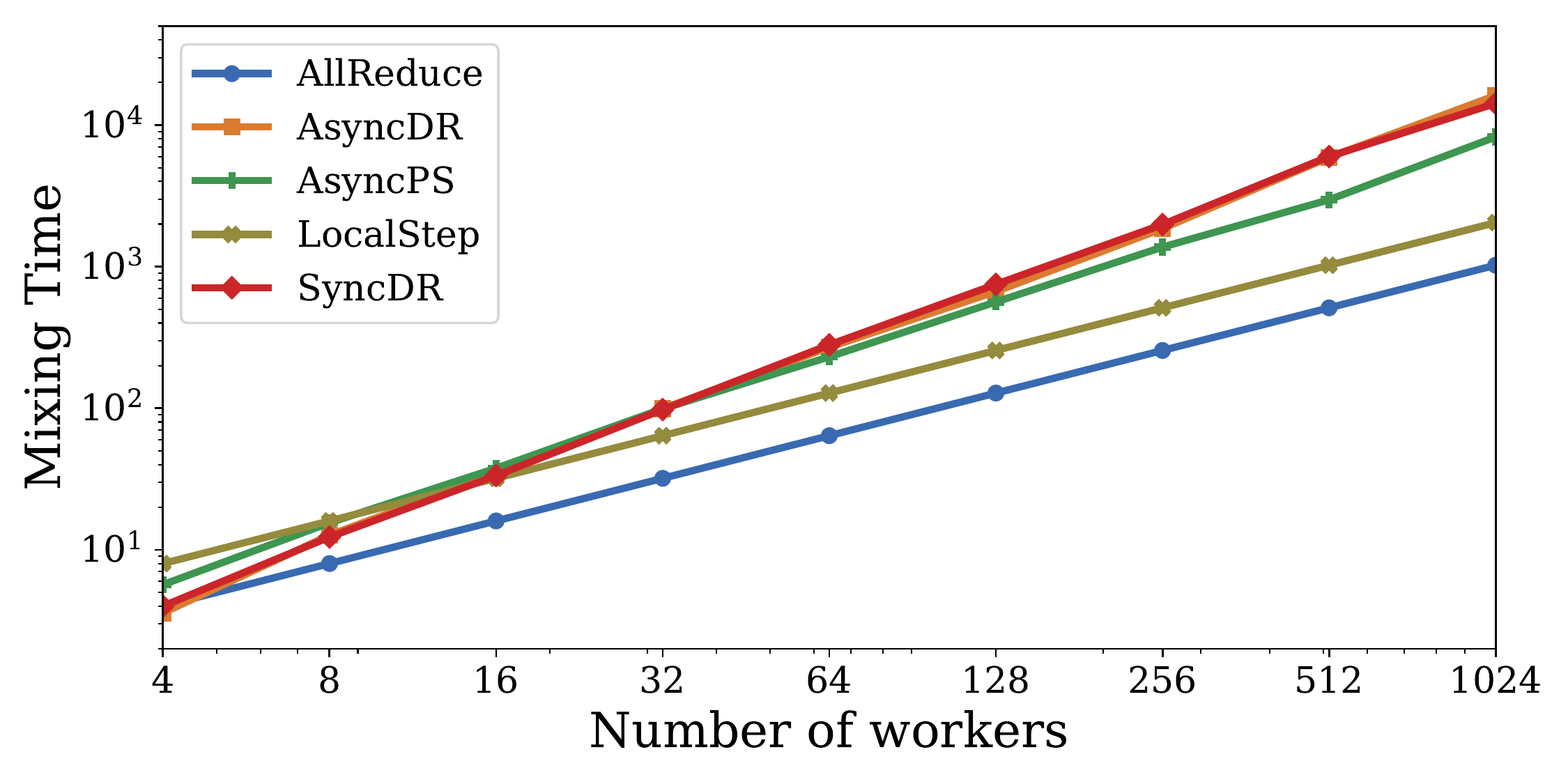}
    \caption{Measurement of Mixing Time among different protocols.}
    \label{exp protocols}
\end{figure}
\section{Experiments}\label{Experiments section}
In this section, we validate our theoretical results.
We first evaluate the mixing time in a variety of protocols that are widely adopted in different algorithms. Then we investigate how different mixing time affects the convergence on an identical problem via training Resnet20 on CIFAR10.

\myparagraph{System Configuration.}
All the models and training scripts in this section are implemented in PyTorch and run on an Ubuntu 16.04 LTS cluster using a SLURM workload manager running CUDA 9.2, configured with 4 NVIDIA Tesla P100 GPUs. We launch one thread as one worker and let them use OpenMPI as the communication backend. 

\myparagraph{Mixing Time in Different Protocols.}
We start from evaluating the mixing time in the following protocols:
\begin{enumerate}[nosep]
    \item \textbf{AllReduce} \cite{gropp1999using}: AllReduce is a centralized protocol from MPI, it allows each worker obtain averaged value from other workers via atomic broadcast and reduce operations. It is straightforward to verify in theory the mixing time of AllReduce is $n$, where $n$ is the number of parallel workers.
    \item \textbf{Local Step} \cite{lin2018don}: This protocol lets workers conduct $m$ optimization steps before stepping into AllReduce. That means, it increases the mixing time of AllReduce by a factor of $m$ and result in a mixing time of $mn$ (we set $m=2$ in our experiment.).
    \item \textbf{Asynchronous Parameter Server (AsyncPS)} \cite{li2014communication}: This protocol is an asynchronous version of Parameter Server implementation, where workers can query the model from a centralized server without global synchronization.
    \item \textbf{Synchronous Ring (SyncDR)} \cite{lian2017can}: This protocol connects all the workers using a ring overlay network and lets workers average parameters with their two neighbors. Synchronous version of this algorithm has a synchronization step that maintains the barrier over workers.
    \item \textbf{Asynchronous Ring (AsyncDR)} \cite{lian2017asynchronous}: This is an asynchronous version of Decentralized Ring, where workers can randomly choose one neighbor to average their parameters with between two adjacent optimization steps.
\end{enumerate}
We run each protocol for 1000 times and plot the results in Figure~\ref{exp protocols}. We can see that in general, as the number of workers increases, the mixing time increases in all of the protocols. Specifically, AllReduce and Local Step increases at a slower rate compared to other protocols as they have synchronization over time and in some sense closer to the perfect communication. On the other hand, we can see that of all the other three protocols, AsyncDR grows quickest compared to SyncDR and AsyncPS as it contains both decentralization and asynchrony, and thus more time or iterations are needed for the workers to "mix" when the system scales up. 

\myparagraph{Convergence under Various Mixing Time.}
We proceed to discuss how convergence of parallel learning can be affected by the mixing time. 
We launch 8 workers, collaborating to train Resnet20 on CIFAR10 using  SGD, and let them communicate via synchronously averaging parameters with all the others. Formally, let $\*x_i$ denote the model on worker $i$ and $n$ denote the number of workers, this communication scheme can be expressed as: $\*x_i\leftarrow \*x_i + \frac{1}{n}\sum_{j=1}^{n}(\*x_j-\*x_i), \forall i\in\{1, \cdots, n\}$. The mixing time for this protocol is $n$.
We now adopt two different ways to reduce the mixing time: 1) \textbf{Reduced Frequency.} As it simply reduced the frequency of workers participate in the communication. A worker is allowed to conduct multiple optimization steps before communication. 2) \textbf{Slack Matrix.} where this method refers to adding a positive parameter $\gamma$ and modify the protocol as: $\*x_i\leftarrow \*x_i +  \frac{\gamma}{n}\sum_{j=1}^{n}(\*x_j-\*x_i), \forall i\in\{1, \cdots, n\}$. In theory, this will increase the mixing time by a factor of $\gamma^{-1}$.
To measure the consensus among workers, we further define \emph{stationary distance} as $\frac{1}{n}\sum_{i=1}^{n}\left\|\*x_i-\*x_\infty\right\|^2$, where $\*x_\infty=\frac{1}{n}\sum_{j=1}^{n}\*x_j$ denotes the consensus of this protocol \cite{boyd2004fastest,lian2017can}.

We plot the results in Figure~\ref{performance}.
We can see that in general, a larger mixing time slows down the rate for workers to reach consensus, and thus affecting the convergence rate. This is aligned with our intuition and theory. On the other hand, of the two methods, we can see that in general, reduced frequency method is more robust than slack matrix method as in the extreme case ($\tmix{}$ is increased by 1000X), reduced frequency can still allow learning to follow a small and less noisy curve. This finding matches several empirical results from previous works \cite{lin2018don,stich2018local}.
An interesting observation is when $\tmix{}$ is large, the Slack Matrix method has much more unstable curve compared to the Reduced Frequency method. This happens because the Slack Matrix method suffers from extra numerical error. Note that when $\tmix{} = 8000$ (and correspondingly, $\gamma$ is approximately $1e-4$), the average distance among workers is around $1e1$, which implies in average each coordinate differs around $1e-5$ (Resnet20 has $0.27M$ params). That implies when multiplied by $\gamma$, each coordinate will be updated by a number around $1e-9$. Because IEEE 754 float only has around $6$ effective decimal digits, this method incurs precision errors here.

%% file: section/Conclusion.tex
\section{Conclusion}
In this paper we propose MixML, a general framework for analyzing weakly consistent parallel learning. We abstract the communication process in parallel training as a sequence of transitions to a state and define a parameter called mixing time that quantifies the communication protocol. 
We provably show in a variety of algorithms how mixing time can be used to easily obtain convergence bound. Our experimental results validate our theory and reveals how mixing time affects the convergence of learning algorithms in a parallel setting.

%% file: section/appendix.tex
\paragraph{Proof to Lemma 1}
\begin{proof}
Let $N=\left\lfloor\frac{t-s}{\tmix{}}\right\rfloor$, we obtain
\begin{align*}
    &\left\|\prod_{k=s}^{t-1}\*M_k\*X-\*M_\infty \*X\right\|\\
    \overset{\text{Assumption~\ref{assume consensus}}}{=} & \left\|\prod_{k=s+N\tmix{}}^{t-1}\*M_k\left(\prod_{k=s}^{s+N\tmix{}-1}\*M_k\*X\right)-\prod_{k=s+N\tmix{}}^{t-1}\*M_k\left(\*M_\infty \*X\right)\right\|\\
    \leq & \left\|\prod_{k=s+N\tmix{}}^{t-1}\*M_k\right\|\left\|\prod_{k=s}^{s+N\tmix{}-1}\*M_k\*X-\*M_\infty \*X\right\|\\
    \overset{\text{Assumption~\ref{assume scale bound}}}{\leq} & \xi\left\|\prod_{k=s}^{s+N\tmix{}-1}\*M_k\*X-\*M_\infty \*X\right\|\\
    \overset{\text{Assumption~\ref{assume consensus}}}{\leq} & \xi\left\|\prod_{k=s+(N-1)\tmix{}}^{s+N\tmix{}-1}\*M_k\left(\prod_{k=s}^{s+(N-1)\tmix{}-1}\*M_k\*X\right)-\prod_{k=s+(N-1)\tmix{}}^{s+N\tmix{}-1}\*M_k\left(\*M_\infty \*X\right)\right\|\\
    \overset{\text{Assumption~\ref{assume mixing time}}}{\leq} & 2^{-1}\xi\left\|\prod_{k=s}^{s+(N-1)\tmix{}-1}\*M_k\*X-\*M_\infty \*X\right\|\\
    \leq & \cdots\\
    \leq & 2^{-N}\xi\left\|\*X-\*M_\infty \*X\right\|\\
    \leq & 2^{-N}\xi(\|\*X\|+\|\*M_\infty\|\|\*X\|)\\
    \overset{\text{Assumption~\ref{assume scale bound}}}{\leq} & 2^{-N}(1+\xi)\xi\left\|\*X\right\|
 \end{align*}
That completes the proof.
\end{proof}

\paragraph{Proof to Lemma 2}
\begin{proof}
Let $\*\Omega_t=\*{\hat{\Pi}}_t(\*{\tilde{\delta}}_t^{(u)}-\*\delta_t^{(u)}), \forall t\geq 0$, 
from the formulation in Section~\ref{MixML section} we obtain
\begin{align*}
    & \mathbb{E}\|\*x_t -  \*u_t\|^2\\
    \overset{\text{Assumption~\ref{assume projection}}}{=} & \mathbb{E}\|\*\Pi_*\*M_\infty \*X_t-\*\Pi_t\*X_t\|^2\\
    \overset{\text{Assumption~\ref{assume consensus}}}{=} & \mathbb{E}\|\*\Pi_t(\*I-\*M_\infty)\*X_t\|^2\\
    \leq & \mathbb{E}\|\*\Pi_t\|\|(\*I-\*M_\infty)\*X_t\|^2\\
    \leq & \mathbb{E}\|(\*I-\*M_\infty)\*X_t\|^2\\
    \overset{\text{Equation~(\ref{Equation prototype})}}{=} & \mathbb{E}\left\|-\sum_{k=0}^{t-1}\prod_{m=k}^{t-1}\*M_{m}\*\Delta_k - \left(-\sum_{k=0}^{t-1}\*M_\infty\*\Delta_{k}\right)\right\|^2\\
    \overset{\text{Assumption~\ref{assume projection}}}{=} & \mathbb{E}\left\|-\sum_{k=0}^{t-1}\prod_{m=k}^{t-1}\*M_{m}\*{\hat{\Pi}}_k\*\delta^{(u)}_{k}+\sum_{k=0}^{t-1}\prod_{m=k}^{t-1}\*M_{m}\*\Omega_{k} - \left(-\sum_{k=0}^{t-1}\*M_\infty\*{\hat{\Pi}}_k\*\delta^{(u)}_{k} + \sum_{k=0}^{t-1}\*M_\infty\*\Omega_{k}\right)\right\|^2\\
    = & \mathbb{E}\left\|\sum_{k=0}^{t-1}\prod_{m=k}^{t-1}\*M_m\*{\hat{\Pi}}_k\*\delta^{(u)}_{k} - \sum_{k=0}^{t-1}\*M_\infty\*{\hat{\Pi}}_k\*\delta^{(u)}_{k}\right\|^2+ \mathbb{E}\left\|\sum_{k=0}^{t-1}\prod_{m=k}^{t-1}\*M_m\*\Omega_{k} - \sum_{k=0}^{t-1}\*M_\infty\*\Omega_{k}\right\|^2\\
    \overset{\text{Assumption~\ref{assumesampler}}}{=} & \mathbb{E}\left(\sum_{k=0}^{t-1}\left\|\prod_{m=k}^{t-1}\*M_m\*{\hat{\Pi}}_k\*\delta^{(u)}_{k} - \*M_\infty\*{\hat{\Pi}}_k\*\delta^{(u)}_{k}\right\|\right)^2 + \mathbb{E}\sum_{k=0}^{t-1}\left\|\prod_{m=k}^{t-1}\*M_m\*\Omega_{k} - \*M_\infty\*\Omega_{k}\right\|^2\\
    \overset{\text{Lemma 1}}{\leq} & (1+\xi)^2\xi^2\mathbb{E}\left(\sum_{k=0}^{t-1}2^{-\left\lfloor\frac{t-k-1}{\tmix{}}\right\rfloor}\|\*{\hat{\Pi}}_k\*\delta^{(u)}_{k}\|\right)^2 + (1+\xi)^2\xi^2\mathbb{E}\sum_{k=0}^{t-1}\left(2^{-\left\lfloor\frac{t-k-1}{\tmix{}}\right\rfloor}\|\*\Omega_{k}\|\right)^2\\
\end{align*}
Summing over from $t=0$ to $T-1$, we obtain
\begin{align*}
    &\sum_{t=0}^{T-1}\alpha_t\mathbb{E}\|\*x_t -  \*u_t\|^2\\
    \leq & (1+\xi)^2\xi^2\mathbb{E}\sum_{t=0}^{T-1}\alpha_t\left(\sum_{k=0}^{t-1}2^{-\left\lfloor\frac{t-k-1}{\tmix{}}\right\rfloor}\|\*{\hat{\Pi}}_k\*\delta^{(u)}_{k}\|\right)^2 + (1+\xi)^2\xi^2\mathbb{E}\sum_{t=0}^{T-1}\alpha_t\sum_{k=0}^{t-1}\left(2^{-\left\lfloor\frac{t-k-1}{\tmix{}}\right\rfloor}\|\*\Omega_{k}\|\right)^2\\
    \overset{\text{Assumption~\ref{assumesampler}}}{\leq} & 4(1+\xi)^2\xi^2\tmix^2{}\sum_{t=0}^{T-1}\alpha_t^3\mathbb{E}\|\*\delta_t^{(u)}\|^2 + 2(1+\xi)^2\xi^2\tmix{}\sigma^2\sum_{t=0}^{T}\alpha_t^3
\end{align*}
Note that
\begin{align*}
    \sum_{t=0}^{T-1}\alpha_t\mathbb{E}\|\*\delta_t^{(x)}-\*\delta_t^{(u)}\|^2 \leq & \sum_{t=0}^{T-1}\alpha_t\mathbb{E}\left(\sum_{k=0}^{t}L_{k,t}\|\*u_k-\*x_k\|\right)^2\\
    = & \mathcal{L}^2\sum_{t=0}^{T-1}\alpha_t\mathbb{E}\left(\sum_{k=0}^{t}\frac{L_{k,t}}{\mathcal{L}}\|\*u_k-\*x_k\|\right)^2\\
    \leq & \mathcal{L}^2\sum_{t=0}^{T-1}\alpha_t\sum_{k=0}^{t}\frac{L_{k,t}}{\mathcal{L}}\mathbb{E}\|\*u_k-\*x_k\|^2\\
    \leq & \mathcal{L}^2\sum_{t=0}^{T-1}\alpha_t\mathbb{E}\|\*u_t-\*x_t\|^2
\end{align*}
Putting our bound for $\sum_{t=0}^{T-1}\alpha_t\mathbb{E}\|\*u_t-\*x_t\|^2$, we obtain
\begin{align*}
    \sum_{t=0}^{T-1}\alpha_t\mathbb{E}\|\*\delta_t^{(x)}-\*\delta_t^{(u)}\|^2 \leq & 4(1+\xi)^2\xi^2\tmix^2{}\mathcal{L}^2\sum_{t=0}^{T-1}\alpha_t^3\mathbb{E}\|\*\delta_t^{(u)}\|^2 + 2(1+\xi)^2\xi^2\tmix{}\mathcal{L}^2\sigma^2\sum_{t=0}^{T}\alpha_t^3\\
    \leq & 8(1+\xi)^2\xi^2\tmix^2{} \mathcal{L}^2\sum_{t=0}^{T-1}\alpha_t^3\mathbb{E}\|\*\delta_t^{(x)}\|^2 + 8(1+\xi)^2\xi^2\tmix^2{} \mathcal{L}^2\sum_{t=0}^{T-1}\alpha_t^3\mathbb{E}\|\*\delta_t^{(u)}-\*\delta_t^{(x)}\|^2\\
    & + 2(1+\xi)^2\xi^2\tmix{}\sigma^2 \mathcal{L}^2\sum_{t=0}^{T}\alpha_t^3
\end{align*}
Rearrange the terms, we obtain
\begin{align*}
    \sum_{t=0}^{T-1}\alpha_t\mathbb{E}\|\*\delta_t^{(x)}-\*\delta_t^{(u)}\|^2 \leq 16(1+\xi)^2\xi^2\tmix^2{}\mathcal{L}^2\sum_{t=0}^{T-1}\alpha_t^3\mathbb{E}\|\*\delta_t^{(x)}\|^2 + 4(1+\xi)^2\xi^2\tmix{}\sigma^2\mathcal{L}^2\sum_{t=0}^{T}\alpha_t^3
\end{align*}
That completes the proof.
\end{proof}

\paragraph{Proof to Lemma 3}
\begin{proof}
Let $\*\zeta_t=\nabla f(\*u_t)-\nabla\tilde{f}(\*u_t)$, by the Lipschitzness assumption, we have
\begin{align*}
    f(\*x_{t+1}) = f(\*x_t-\alpha_t\nabla f(\*u_k)+\alpha_t\*\zeta_t) \leq f(\*x_t) - \alpha_t\left\langle\nabla f(\*x_t), \nabla f(\*u_k)-\*\zeta_t\right\rangle + \frac{\alpha_t^2L}{2}\left\|\nabla f(\*u_k)-\*\zeta_t\right\|^2
\end{align*}
Taking expectation with respect to $\sigma$-field on both sides, we haue
\begin{align*}
    \mathbb{E}\left[f(\*x_{t+1})|\mathcal{F}_t\right]
    \leq & f(\*x_t) - \alpha_t\mathbb{E}\left[\left\langle\nabla f(\*x_t), \nabla f(\*u_t)-\*\zeta_t\right\rangle|\mathcal{F}_t\right] + \frac{\alpha_t^2L}{2}\mathbb{E}\left[\left\|\nabla f(\*u_t)-\*\zeta_t\right\|^2|\mathcal{F}_t\right]\\
    \leq & f(\*x_t) - \alpha_t\mathbb{E}\left[\left\langle\nabla f(\*x_t), \nabla f(\*u_t)\right\rangle|\mathcal{F}_t\right] + \frac{\alpha_t^2L}{2}\mathbb{E}\left[\left\|\*\zeta_t\right\|^2|\mathcal{F}_t\right] + \frac{\alpha_t^2L}{2}\mathbb{E}\left[\left\|\nabla f(\*u_t)\right\|^2|\mathcal{F}_t\right]\\
    \leq & f(\*x_t) - \frac{\alpha_t}{2}\|\nabla f(\*x_t)\|^2 - \frac{\alpha_t}{2}\mathbb{E}\left[\|\nabla f(\*u_t)\|^2|\mathcal{F}_t\right] + \frac{\alpha_t}{2}\mathbb{E}\left[\|\nabla f(\*x_t) - \nabla f(\*u_t)\|^2|\mathcal{F}_t\right]\\
        & + \frac{\alpha_t^2L}{2}\sigma^2 + \frac{\alpha_t^2L}{2}\mathbb{E}\left[\left\|\nabla f(\*u_t)\right\|^2|\mathcal{F}_t\right]
\end{align*}
rearrange the terms, we obtain
\begin{align*}
    2\mathbb{E}\left[f(\*x_{t+1})|\mathcal{F}_t\right] \leq 2f(\*x_t) - \alpha_t\|\nabla f(\*x_t)\|^2 + \alpha_t\left\|\nabla f(\*x_t) - \nabla f(\*u_t)\right\|^2 + \alpha_t^2L\sigma^2
\end{align*}
Take full expectation we haue
\begin{align*}
    2\mathbb{E}f(\*x_{t+1}) \leq 2\mathbb{E}f(\*x_t) - \alpha_t\mathbb{E}\|\nabla f(\*x_t)\|^2 + \alpha_t\left\|\nabla f(\*x_t) - \nabla f(\*u_t)\right\|^2 + \alpha_t^2L\sigma^2
\end{align*}
Summing from $t=0$ to $T-1$ we obtain
\begin{align*}
    \sum_{t=0}^{T-1}\alpha_t\mathbb{E}\left\|\nabla f(\*x_t)\right\|^2\leq 2(\mathbb{E}f(\*x_0)-\mathbb{E}f(\*x_T)) + \sum_{t=0}^{T-1}\alpha_t^2\sigma^2L + \sum_{t=0}^{T-1}\alpha_t\mathbb{E}\|\nabla f(\*x_t) -  \nabla f(\*u_t)\|^2
\end{align*}
That completes the proof.
\end{proof}

\paragraph{Proof to Theorem 1.}
From Lemma 3 we know
\begin{align*}
    \sum_{t=0}^{T-1}\alpha_t\mathbb{E}\left\|\nabla f(\*x_t)\right\|^2 \leq 2(f(\*0)-f(\*x_T)) + \frac{\sigma^2L}{2}\sum_{t=0}^{T-1}\alpha_t^2+ \sum_{t=0}^{T-1}\alpha_t\mathbb{E}\|\nabla f(\*x_t)-\nabla f(\*u_t)\|^2
\end{align*}
By fitting $\mathcal{L}=1$ and $\*\delta_t^{(x)}=\nabla f(\*x_t)$ in Lemma 2, we obtain
\begin{align*}
    \sum_{t=0}^{T-1}\alpha_t\left(1-16(1+\xi)^2\xi^2\tmix^2{}L^2\right)\mathbb{E}\left\|\nabla f(\*x_t)\right\|^2 \leq 2(f(\*0)-f(\*x_T)) + \frac{\sigma^2L}{2}\sum_{t=0}^{T-1}\alpha_t^2 + 4(1+\xi)^2\xi^2\tmix{}\sigma^2L^2\sum_{t=0}^{T}\alpha_t^3
\end{align*}

\paragraph{Proof to Lemma 4}
\begin{proof}
Let $\*x_{t,j}$ denote the $j$-th coordinate of $\*x_t$. First, consider the following function
\begin{align*}
\*\delta_{t,j} = \frac{\*m_{t,j}}{\*v_{t,j}^p} = \frac{\beta_1\*m_{t-1,j} + (1-\beta_1)\*g_{t,j}}{(\beta_2\*v_{t-1,j}+(1-\beta_2)\*g_{t,j}^2)^p} = \frac{(1-\beta_1)\sum_{k=0}^{t}\beta_1^{t-k}\*g_{k,j}}{\left((1-\beta_2)\sum_{k=0}^{t}\beta_2^{t-k}\*g_{k,j}^2\right)^p}
\end{align*}
By taking the coordinate-wise derivative, we obtain
\begin{align*}
    \left|\frac{\partial}{\partial \*g_{k,j}}\*\delta_{t,j}\right| =& \left|\frac{\left(\frac{\partial \*m_{t,j}}{\partial \*g_{k,j}}\right)\*v_{t,j}^p - p\*v_{t,j}^{p-1}\left(\frac{\partial \*v_{t,j}}{\partial \*g_{k,j}}\right)\*m_{t,j}}{\*v_{t,j}^{2p}}\right|\\
    = & \left|\frac{\left(\frac{\partial \*m_{t,j}}{\partial \*g_{k,j}}\right)}{\*v_{t,j}^p} - \frac{p\left(\frac{\partial \*v_{t,j}}{\partial \*g_{k,j}}\right)\*m_{t,j}}{\*v_{t,j}^{p+1}}\right|\\
    = & \left|\frac{(1-\beta_1)\beta_1^{t-k}}{\*v_{t,j}^p} - \frac{2p(1-\beta_2)\beta_2^{t-k}\*m_{t,j}}{\*v_{t,j}^{p+1}}\right|\\
    \leq & \left|\frac{(1-\beta_1)\beta_1^{t-k}}{\*v_{t,j}^p}\right| + \left|\frac{2p(1-\beta_2)\beta_2^{t-k}\*m_{t,j}}{\*v_{t,j}^{p+1}}\right|\\
    \leq & \frac{(1-\beta_1)\beta_1^{t-k}}{c^p} + \frac{2p(1-\beta_2)\beta_2^{t-k}G_\infty}{c^{p+1}}
\end{align*}
That being said,
\begin{align*}
    & \|\*\delta_t(\*u_t, \cdots, \*u_0) - \*\delta_t(\*x_t, \cdots, \*x_0)\|\\
    \leq & \sum_{k=0}^{t}\|\*\delta_t(\*x_t, \cdots, \*x_{k+1}, \*u_k, \*u_{k-1}, \cdots, \*u_0) - \*\delta_t(\*x_t, \cdots, \*x_{k+1}, \*x_k, \*u_{k-1}, \cdots, \*u_0)\|\\
    = & \sum_{k=0}^{t}\sqrt{\sum_{j=1}^{d}|\*\delta_t(\*x_t, \cdots, \*x_{k+1}, \*u_k, \*u_{k-1}, \cdots, \*u_0)\*e_j - \*\delta_t(\*x_t, \cdots, \*x_{k+1}, \*x_k, \*u_{k-1}, \cdots, \*u_0)\*e_j|^2}\\
    \leq & \sum_{k=0}^{t}\left(\frac{(1-\beta_1)\beta_1^{t-k}}{c^p} + \frac{2p(1-\beta_2)\beta_2^{t-k}G_\infty}{c^{p+1}}\right)\sqrt{\sum_{j=1}^{d}|\*g_k(\*u_k)\*e_j - \*g_k(\*x_k)\*e_j|^2}\\
    \leq & \sum_{k=0}^{t}\left(\frac{(1-\beta_1)\beta_1^{t-k}}{c^p} + \frac{2p(1-\beta_2)\beta_2^{t-k}G_\infty}{c^{p+1}}\right)\|\*g_k(\*u_k)-\*g_k(\*x_k)\|
\end{align*}
Taking square on both sides of the inequality, we obtain
\begin{align*}
    & \|\*\delta_t(\*u_t, \cdots, \*u_0) - \*\delta_t(\*x_t, \cdots, \*x_0)\|^2\\
    \leq & \left(\sum_{k=0}^{t}\left(\frac{(1-\beta_1)\beta_1^{t-k}}{c^p} + \frac{2p(1-\beta_2)\beta_2^{t-k}G_\infty}{c^{p+1}}\right)\|\*g_k(\*u_k)-\*g_k(\*x_k)\|\right)^2\\
    \leq & \frac{1}{M^2}\left(\sum_{k=0}^{t}\left(\frac{(1-\beta_1)\beta_1^{t-k}M}{c^p} + \frac{2p(1-\beta_2)\beta_2^{t-k}G_\infty M}{c^{p+1}}\right)\|\*g_k(\*u_k)-\*g_k(\*x_k)\|\right)^2\\
    \leq & \sum_{k=0}^{t}\left(\frac{(1-\beta_1)\beta_1^{t-k}}{Mc^p} + \frac{2p(1-\beta_2)\beta_2^{t-k}G_\infty}{Mc^{p+1}}\right)\|\*g_k(\*u_k)-\*g_k(\*x_k)\|^2\\
    \leq & \sum_{k=0}^{t}\left(\frac{(1-\beta_1)\beta_1^{t-k}}{Mc^p} + \frac{2p(1-\beta_2)\beta_2^{t-k}G_\infty}{Mc^{p+1}}\right)L^2\|\*u_k-\*x_k\|^2
\end{align*}
where $M=\min\left\{\frac{c^p}{2}, \frac{c^{p+1}}{4pG_\infty}\right\}$.
And
\begin{align*}
    \sum_{k=0}^{\infty}\left(\frac{(1-\beta_1)\beta_1^{t-k}}{Mc^p} + \frac{2p(1-\beta_2)\beta_2^{t-k}G_\infty}{Mc^{p+1}}\right)L^2 = \frac{c+2pG\infty}{Mc^{p+1}}=\frac{c+2pG\infty}{Mc^{2p+1}}\max\left\{2,\frac{4pG_\infty}{c}\right\}
\end{align*}
That completes the proof.
\end{proof}

\paragraph{Proof to Theorem 2.}
\begin{proof}
We start from the definitions on the following term
\begin{align*}
    \*h_t = \*x_t + \lambda(\*x_t-\*x_{t-1})
\end{align*}
where $\lambda=\frac{\beta_1}{1-\beta_1}$. Note that the update rule for SAM is 
\begin{align*}
    \*x_{t+1} = \*x_t - \alpha_t\*{\tilde{\delta}}_t^{(u)}
\end{align*}
As a result,
\begin{align*}
    \*h_{t+1} - \*h_t = & -\alpha_t\*{\tilde{\delta}}_t^{(u)} + \lambda(-\alpha_t\*{\tilde{\delta}}_t^{(u)}) - \lambda(-\alpha_{t-1}\*{\tilde{\delta}}_{t-1}^{(u)})\\
    = & \underbrace{-(1+\lambda)\alpha_t\*\delta_t^{(x)} + \lambda\alpha_{t-1}\*\delta_{t-1}^{(x)}}_{T_s} + \underbrace{(1+\lambda)\alpha_t(\*\delta_t^{(x)}-\*{\tilde{\delta}}_t^{(u)}) + \lambda\alpha_{t-1}(\*{\tilde{\delta}}_{t-1}^{(u)}-\*\delta_{t-1}^{(x)})}_{T_p}
\end{align*}
Using Taylor's theorem we obtain
\begin{equation}\label{equation taylor}
\begin{aligned}
    f(\*h_{t+1}) \leq & f(\*h_t) + \langle \nabla f(\*h_t), \*h_{t+1} - \*h_t\rangle + \frac{L}{2}\|\*h_{t+1} - \*h_t\|^2\\
    \leq & f(\*h_t) + \langle \nabla f(\*h_t), T_s\rangle + L\|T_s\|^2 +\langle \nabla f(\*h_t), T_p\rangle + L\|T_p\|^2
\end{aligned}
\end{equation}
Note that the last two terms in the RHS are from parallelism, and we proceed to provide the bound for the last two terms, let $\gamma>0$ be a constant which will be assigned later:
\begin{align*}
    \nabla f(\*h_t)^\top\mathbb{E}\left[T_p\right] = & (1+\lambda)\alpha_t\nabla f(\*h_t)^\top\mathbb{E}\left[\*\delta_t^{(x)}-\*{\tilde{\delta}}_t^{(u)}\right] + \lambda\alpha_{t-1}\nabla f(\*h_t)^\top\mathbb{E}\left[\*\delta_{t-1}^{(x)}-\*{\tilde{\delta}}_{t-1}^{(u)}\right]\\
    = & (1+\lambda)\alpha_t\nabla f(\*h_t)^\top\mathbb{E}\left[\*\delta_t^{(x)}-\*\delta_t^{(u)}\right] + \lambda\alpha_{t-1}\nabla f(\*h_t)^\top\mathbb{E}\left[\*\delta_{t-1}^{(x)}-\*\delta_{t-1}^{(u)}\right]\\
    = & (1+\lambda)\alpha_t\nabla f(\*x_t)^\top\mathbb{E}\left[\*\delta_t^{(x)}-\*\delta_t^{(u)}\right] + \lambda\alpha_{t-1}\nabla f(\*x_t)^\top\mathbb{E}\left[\*\delta_{t-1}^{(x)}-\*\delta_{t-1}^{(u)}\right]\\
    & +  (1+\lambda)\alpha_t(\nabla f(\*h_t)-\nabla f(\*x_t))^\top\mathbb{E}\left[\*\delta_t^{(x)}-\*\delta_t^{(u)}\right] + \lambda\alpha_{t-1}(\nabla f(\*h_t)-\nabla f(\*x_t))^\top\mathbb{E}\left[\*\delta_{t-1}^{(x)}-\*\delta_{t-1}^{(u)}\right]\\
    \leq  & (1+\lambda)\alpha_t\mathbb{E}\|\nabla f(\*x_t)\|\|\*\delta_t^{(x)}-\*\delta_t^{(u)}\| + \lambda\alpha_{t-1}\mathbb{E}\|\nabla f(\*x_t)\|\|\*\delta_{t-1}^{(x)}-\*\delta_{t-1}^{(u)}\|\\
    & + (1+\lambda)\alpha_t\mathbb{E}\|\nabla f(\*h_t)-\nabla f(\*x_t)\|\|\*\delta_t^{(x)}-\*\delta_t^{(u)}\| + \lambda\alpha_{t-1}\mathbb{E}\|\nabla f(\*h_t)-\nabla f(\*x_t)\|\|\*\delta_{t-1}^{(x)}-\*\delta_{t-1}^{(u)}\|\\
    \leq & \frac{(1+\lambda)\alpha_t\gamma}{2}\mathbb{E}\|\nabla f(\*x_t)\|^2 + \frac{(1+\lambda)\alpha_t}{2\gamma}\mathbb{E}\|\*\delta_t^{(x)}-\*\delta_t^{(u)}\|^2 + \frac{\lambda\alpha_{t-1}\gamma}{2}\mathbb{E}\|\nabla f(\*x_t)\|^2 + \frac{\lambda\alpha_{t-1}}{2\gamma}\mathbb{E}\|\*\delta_{t-1}^{(x)}-\*\delta_{t-1}^{(u)}\|^2\\
    & + \frac{(1+\lambda)\alpha_t}{2}\mathbb{E}\|\nabla f(\*h_t)-\nabla f(\*x_t)\|^2 + \frac{(1+\lambda)\alpha_t}{2}\mathbb{E}\|\*\delta_t^{(x)}-\*\delta_t^{(u)}\|^2 + \frac{\lambda\alpha_{t-1}}{2}\mathbb{E}\|\nabla f(\*h_t)-\nabla f(\*x_t)\|^2\\
    & + \frac{\lambda\alpha_{t-1}}{2}\mathbb{E}\|\*\delta_{t-1}^{(x)}-\*\delta_{t-1}^{(u)}\|^2\\
    = & \frac{(1+2\lambda)\alpha_t\gamma}{2}\mathbb{E}\|\nabla f(\*x_t)\|^2 + \frac{(1+2\lambda)\lambda^2\alpha_tL^2}{2}\mathbb{E}\|\alpha_{t-1}\*{\tilde{\delta}}_{t-1}^{(u)}\|^2 + \frac{(1+\lambda)\alpha_t}{2}(\gamma^{-1}+1)\mathbb{E}\|\*\delta_t^{(x)}-\*\delta_t^{(u)}\|^2\\
    & + \frac{\lambda\alpha_{t-1}}{2}(\gamma^{-1}+1)\mathbb{E}\|\*\delta_{t-1}^{(x)}-\*\delta_{t-1}^{(u)}\|^2\\
    = & \frac{(1+2\lambda)\alpha_t\gamma}{2}\mathbb{E}\|\nabla f(\*x_t)\|^2 + \frac{(1+2\lambda)\lambda^2\alpha_t\alpha_{t-1}^2L^2}{2}\mathbb{E}\|\*\delta_{t-1}^{(u)}\|^2 + \frac{(1+2\lambda)\lambda^2\alpha_t\alpha_{t-1}^2L^2}{2}\sigma^2\\
    & + \frac{(1+\lambda)\alpha_t}{2}(\gamma^{-1}+1)\mathbb{E}\|\*\delta_t^{(x)}-\*\delta_t^{(u)}\|^2 + \frac{\lambda\alpha_{t-1}}{2}(\gamma^{-1}+1)\mathbb{E}\|\*\delta_{t-1}^{(x)}-\*\delta_{t-1}^{(u)}\|^2\\
    \leq & \frac{(1+2\lambda)\alpha_t\gamma}{2}\mathbb{E}\|\nabla f(\*x_t)\|^2 + (1+2\lambda)\lambda^2\alpha_t\alpha_{t-1}^2L^2\mathbb{E}\|\*\delta_{t-1}^{(x)}\|^2 + \frac{(1+2\lambda)\lambda^2\alpha_t\alpha_{t-1}^2L^2}{2}\sigma^2\\
    & + \frac{(1+\lambda)\alpha_t}{2}(\gamma^{-1}+1)\mathbb{E}\|\*\delta_t^{(x)}-\*\delta_t^{(u)}\|^2 + \frac{(1+\lambda)\alpha_{t-1}}{2}(\gamma^{-1}+1)\mathbb{E}\|\*\delta_{t-1}^{(x)}-\*\delta_{t-1}^{(u)}\|^2
\end{align*}
Summing from $t=0$ to $T-1$, we obtain
\begin{align*}
    \sum_{t=0}^{T-1}\nabla f(\*h_t)^\top\mathbb{E}[T_p] \leq & \frac{(1+2\lambda)\gamma}{2}\sum_{t=0}^{T-1}\alpha_t\mathbb{E}\|\nabla f(\*x_t)\|^2 + (1+2\lambda)\lambda^2L^2\sum_{t=0}^{T-1}\alpha_t^3\mathbb{E}\|\*\delta_t^{(x)}\|^2\\
    & + (1+\lambda)(\gamma^{-1}+1)\sum_{t=0}^{T-1}\alpha_t\mathbb{E}\|\*\delta_t^{(x)}-\*\delta_t^{(u)}\|^2 + \frac{(1+2\lambda)\lambda^2L^2}{2}\sum_{t=0}^{T-1}\alpha_t^3\sigma^2
\end{align*}
We proceed to bound the second term, specifically,
\begin{align*}
    L\mathbb{E}\left\|T_p\right\|^2 = & L\mathbb{E}\|(1+\lambda)\alpha_t(\*\delta_t^{(x)}-\*{\tilde{\delta}}_t^{(u)}) + \lambda\alpha_{t-1}(\*{\tilde{\delta}}_{t-1}^{(u)}-\*\delta_{t-1}^{(x)})\|^2\\
    = & L\mathbb{E}\|(1+\lambda)\alpha_t(\*\delta_t^{(x)}-\*\delta_t^{(u)}) + \lambda\alpha_{t-1}(\*\delta_{t-1}^{(u)}-\*\delta_{t-1}^{(x)})\|^2 + (1+\lambda)^2\alpha_t^2\sigma^2L + \lambda^2\alpha_{t-1}^2\sigma^2L\\
    \leq & 2(1+\lambda)^2\alpha_t^2L\mathbb{E}\|\*\delta_t^{(x)}-\*\delta_t^{(u)}\|^2 + 2\lambda\alpha_{t-1}^2L\mathbb{E}\|\*\delta_{t-1}^{(u)}-\*\delta_{t-1}^{(x)}\|^2 + (1+\lambda)^2\alpha_t^2\sigma^2L + \lambda^2\alpha_{t-1}^2\sigma^2L
\end{align*}
Summing from $t=0$ to $T-1$, we obtain
\begin{align*}
    L\sum_{t=0}^{T-1}\mathbb{E}\left\|T_p\right\|^2\leq 2((1+\lambda)^2+\lambda)L\sum_{t=0}^{T-1}\alpha_t^2\mathbb{E}\|\*\delta_t^{(x)}-\*\delta_t^{(u)}\| + ((1+\lambda)^2+\lambda^2)\sigma^2L\sum_{t=0}^{T-1}\alpha_t^2
\end{align*}
Combine them together, we obtain
\begin{align*}
    & \sum_{t=0}^{T-1}\nabla f(\*h_t)^\top\mathbb{E}[T_p] + L\sum_{t=0}^{T-1}\mathbb{E}\left\|T_p\right\|^2\\
    \leq & \frac{(1+2\lambda)\gamma}{2}\sum_{t=0}^{T-1}\alpha_t\mathbb{E}\|\nabla f(\*x_t)\|^2 + (1+2\lambda)\lambda^2L^2\sum_{t=0}^{T-1}\alpha_t^3\mathbb{E}\|\*\delta_t^{(x)}\|^2 + (1+\lambda)(\gamma^{-1}+1)\sum_{t=0}^{T-1}\alpha_t\mathbb{E}\|\*\delta_t^{(x)}-\*\delta_t^{(u)}\|^2\\
        &  + \frac{(1+2\lambda)\lambda^2L^2}{2}\sum_{t=0}^{T-1}\alpha_t^3\sigma^2+ 2((1+\lambda)^2+\lambda)L\sum_{t=0}^{T-1}\alpha_t^2\mathbb{E}\|\*\delta_t^{(x)}-\*\delta_t^{(u)}\| + ((1+\lambda)^2+\lambda^2)\sigma^2L\sum_{t=0}^{T-1}\alpha_t^2\\
    \leq & \frac{(1+2\lambda)\gamma}{2}\sum_{t=0}^{T-1}\alpha_t\mathbb{E}\|\nabla f(\*x_t)\|^2 + (1+2\lambda)\lambda^2L^2\sum_{t=0}^{T-1}\alpha_t^3\mathbb{E}\|\*\delta_t^{(x)}\|^2 + (2+\lambda)(\gamma^{-1}+1)\sum_{t=0}^{T-1}\alpha_t\mathbb{E}\|\*\delta_t^{(x)}-\*\delta_t^{(u)}\|^2\\
        & + 2(1+\lambda)^2\sigma^2L\sum_{t=0}^{T-1}\alpha_t^2
\end{align*}
where in the first inequality we fit the sampling variance in the second term into the last term. Note that from Lemma~\ref{lemma general bound} we have
\begin{align*}
    \sum_{t=0}^{T-1}\alpha_t\mathbb{E}\|\*\delta_t^{(x)}-\*\delta_t^{(u)}\|^2 \leq 16(1+\xi)^2\xi^2\tmix^2{}\mathcal{L}^2\sum_{t=0}^{T-1}\alpha_t^3\mathbb{E}\|\*\delta_t^{(x)}\|^2 + 4(1+\xi)^2\xi^2\tmix{}\sigma^2\mathcal{L}^2\sum_{t=0}^{T}\alpha_t^3
\end{align*}
As a result,
\begin{align*}
    & \sum_{t=0}^{T-1}\nabla f(\*h_t)^\top\mathbb{E}[T_p] + L\sum_{t=0}^{T-1}\mathbb{E}\left\|T_p\right\|^2\\
    \leq & \frac{(1+2\lambda)\gamma}{2}\sum_{t=0}^{T-1}\alpha_t\mathbb{E}\|\nabla f(\*x_t)\|^2 + \left((1+2\lambda)\lambda^2L^2+16(1+\xi)^2\xi^2\tmix^2{}\mathcal{L}^2(2+\lambda)(\gamma^{-1}+1)\right)\sum_{t=0}^{T-1}\alpha_t^3\mathbb{E}\|\*\delta_t^{(x)}\|^2\\
        & + 2(1+\lambda)^2L\sum_{t=0}^{T-1}\alpha_t^2\sigma^2 + 4(1+\xi)^2\xi^2\tmix{}\mathcal{L}^2(2+\lambda)(\gamma^{-1}+1)\sum_{t=0}^{T-1}\alpha_t^3\sigma^2\\
    \leq & \frac{(1+2\lambda)\gamma}{2}\sum_{t=0}^{T-1}\alpha_t\mathbb{E}\|\nabla f(\*x_t)\|^2\\
        & + \left((1+2\lambda)\lambda^2L^2+16(1+\xi)^2\xi^2\tmix^2{}\mathcal{L}^2(2+\lambda)(\gamma^{-1}+1)\right)c^{-2p}\sum_{t=0}^{T-1}\alpha_t^3\mathbb{E}\left\|(1-\beta_1)\sum_{k=0}^{t}\beta_1^{k-t}\nabla f(\*x_k)\right\|^2\\
        & + 2(1+\lambda)^2L\sum_{t=0}^{T-1}\alpha_t^2\sigma^2 + 4(1+\xi)^2\xi^2\tmix{}\mathcal{L}^2(2+\lambda)(\gamma^{-1}+1)\sum_{t=0}^{T-1}\alpha_t^3\sigma^2\\
    \leq & \frac{(1+2\lambda)\gamma}{2}\sum_{t=0}^{T-1}\alpha_t\mathbb{E}\|\nabla f(\*x_t)\|^2 + \left((1+2\lambda)\lambda^2L^2+16(1+\xi)^2\xi^2\tmix^2{}\mathcal{L}^2(2+\lambda)(\gamma^{-1}+1)\right)c^{-2p}\sum_{t=0}^{T-1}\alpha_t^3\mathbb{E}\left\|\nabla f(\*x_t)\right\|^2\\
        & + 2(1+\lambda)^2L\sum_{t=0}^{T-1}\alpha_t^2\sigma^2 + 4(1+\xi)^2\xi^2\tmix{}\mathcal{L}^2(2+\lambda)(\gamma^{-1}+1)\sum_{t=0}^{T-1}\alpha_t^3\sigma^2
\end{align*}
That completes the proof.
\end{proof}
From Equation~\ref{equation taylor} we know
\begin{align*}
    f(\*h_{t+1}) \leq & f(\*h_t) + \langle \nabla f(\*h_t), \*h_{t+1} - \*h_t\rangle + \frac{L}{2}\|\*h_{t+1} - \*h_t\|^2\\
    \leq & f(\*h_t) + \langle \nabla f(\*h_t), T_s\rangle + L\|T_s\|^2 +\underbrace{\langle \nabla f(\*h_t), T_p\rangle + L\|T_p\|^2}_{U_t}
\end{align*}
Combined with the result of the sequential proof from \cite{zhou2018convergence}, we obtain
\begin{align*}
    & \sum_{t=0}^{T-1}\alpha_t\mathbb{E}\|\nabla f(\*x_t)\|^2G_\infty^{-2p}\\
    \leq & \mathbb{E}f\left(\frac{\*x_0}{1-\beta_1}\right) - \mathbb{E}f^* + 3L\sum_{t=0}^{T-1}\mathbb{E}\|\alpha_t\*V_t^{-p}\*g_t\|^2 + 6L\left(\frac{\beta_1}{1-\beta_1}\right)^2\sum_{t=0}^{T-1}\mathbb{E}\|\alpha_{t-1}\*V_{t-1}^{-p}\*m_{t-1}\|^2 + \sum_{t=0}^{T-1}U_t\\
    \leq & \mathbb{E}f\left(\frac{\*x_0}{1-\beta_1}\right) - \mathbb{E}f^* + \left(3L + \frac{6L\beta_1^2}{1-\beta_1}\right)\frac{G_\infty^{1+q-4p}}{(1-\beta_2)^{2p}\left(1-\frac{\beta_1}{\beta_2^{2p}}\right)}\mathbb{E}\left[\sum_{t=0}^{T-1}\sum_{j=1}^{d}\alpha_t^2|\*g_{t,j}|^{1-q}\right] + \sum_{t=0}^{T-1}U_t\\
    \leq & \mathbb{E}f\left(\frac{\*x_0}{1-\beta_1}\right) - \mathbb{E}f^* + \left(3L + \frac{6L\beta_1^2}{1-\beta_1}\right)\frac{G_\infty^{2-4p}d}{(1-\beta_2)^{2p}\left(1-\frac{\beta_1}{\beta_2^{2p}}\right)}\sum_{t=0}^{T-1}\alpha_t^2\\
        & + \frac{(1+2\lambda)\gamma}{2}\sum_{t=0}^{T-1}\alpha_t\mathbb{E}\|\nabla f(\*x_t)\|^2 + \left((1+2\lambda)\lambda^2L^2+16(1+\xi)^2\xi^2\tmix^2{}\mathcal{L}^2(2+\lambda)(\gamma^{-1}+1)\right)c^{-2p}\sum_{t=0}^{T-1}\alpha_t^3\mathbb{E}\left\|\nabla f(\*x_t)\right\|^2\\
        & + 2(1+\lambda)^2L\sum_{t=0}^{T-1}\alpha_t^2\sigma^2 + 4(1+\xi)^2\xi^2\tmix{}\mathcal{L}^2(2+\lambda)(\gamma^{-1}+1)\sum_{t=0}^{T-1}\alpha_t^3\sigma^2
\end{align*}
Let $\gamma^{-1}=2(1+2\lambda)G^{2p}$, rearrange the terms, we obtain
\begin{align*}
    & \sum_{t=0}^{T-1}\alpha_t\mathbb{E}\|\nabla f(\*x_t)\|^2\\
    \leq & 2G_\infty^{2p}\left(\mathbb{E}f\left(\frac{\*x_0}{1-\beta_1}\right) - \mathbb{E}f^*\right) + 2\left(3L + \frac{6L\beta_1^2}{1-\beta_1}\right)\frac{G_\infty^{2-2p}d}{(1-\beta_2)^{2p}\left(1-\frac{\beta_1}{\beta_2^{2p}}\right)}\sum_{t=0}^{T-1}\alpha_t^2 + 4(1+\lambda)^2LG_\infty^{2p}\sum_{t=0}^{T-1}\alpha_t^2\sigma^2\\
        & + 8(1+\xi)^2\xi^2\tmix{}\mathcal{L}^2(2+\lambda)(2(1+2\lambda)G_\infty^{2p}+1)G_\infty^{2p}\sum_{t=0}^{T-1}\alpha_t^3\sigma^2\\
    \leq & 2G_\infty^{2p}\left(\mathbb{E}f\left(\frac{\*x_0}{1-\beta_1}\right) - \mathbb{E}f^*\right) + 2\left(3L + \frac{6L\beta_1^2}{1-\beta_1}\right)\frac{G_\infty^{2-2p}d}{(1-\beta_2)^{2p}\left(1-\frac{\beta_1}{\beta_2^{2p}}\right)}\sum_{t=0}^{T-1}\alpha_t^2 + \frac{4LG_\infty^{2p}\sigma^2}{(1-\beta_1)^2}\sum_{t=0}^{T-1}\alpha_t^2\\
        & + 8(1+\xi)^2\xi^2\tmix{}\frac{(c+2pG_\infty)^2L^4G_\infty^{2p}}{c^{4p+2}}\left(\frac{2(1+\beta_1)(2-\beta_1)G_\infty^{2p}}{(1-\beta_1)^2} + \frac{2-\beta_1}{1-\beta_1}\right)\max\left\{4, \frac{16p^2G_\infty^2}{c^2}\right\}\sum_{t=0}^{T-1}\alpha_t^3\sigma^2
\end{align*}
That completes the proof.

And the coefficients in the original rate shown in the paper are:
\begin{align*}
    C_1 = & 2G_\infty^{2p}\left(\mathbb{E}f\left(\frac{\*x_0}{1-\beta_1}\right) - \mathbb{E}f^*\right)\\
    C_2 = & 2\left(3L + \frac{6L\beta_1^2}{1-\beta_1}\right)\frac{G_\infty^{2-2p}d}{(1-\beta_2)^{2p}\left(1-\frac{\beta_1}{\beta_2^{2p}}\right)}\\\
    C_3 = & \frac{4LG_\infty^{2p}}{(1-\beta_1)^2}\\
    C_4 = & 8(1+\xi)^2\xi^2\frac{(c+2pG_\infty)^2L^4G_\infty^{2p}}{c^{4p+2}}\left(\frac{2(1+\beta_1)(2-\beta_1)G_\infty^{2p}}{(1-\beta_1)^2} + \frac{2-\beta_1}{1-\beta_1}\right)\max\left\{4, \frac{16p^2G_\infty^2}{c^2}\right\}
\end{align*}

\begin{lemma}\label{lemmasequence}
Given $0\leq\rho<1$ and $T$, a positive integer. Also given non-negative sequences $\{a_t\}_{t=1}^{\infty}$ and $\{b_t\}_{t=1}^{\infty}$ with $\{a_t\}_{t=1}^{\infty}$ being non-increasing, the following inequalities holds:
\begin{align*}
    \sum_{t=1}^{k}a_t\left(\sum_{s=1}^{t}\rho^{-\left\lfloor\frac{t-s}{T}\right\rfloor}b_s\right) \leq & \frac{T}{1-\rho}\sum_{s=1}^{k}a_sb_s\\
    \sum_{t=1}^{k}a_t\left(\sum_{s=1}^{t}\rho^{-\left\lfloor\frac{t-s}{T}\right\rfloor}b_s\right)^2 \leq & \frac{T^2}{(1-\rho)^2}\sum_{s=1}^{k}a_sb_s^2
\end{align*}
\end{lemma}
\begin{proof}
Firstly,
\begin{align*}
    S_k = \sum_{t=1}^{k}a_t\left(\sum_{s=1}^{t}\rho^{-\left\lfloor\frac{t-s}{T}\right\rfloor}b_s\right) = \sum_{s=1}^{k}\sum_{t=s}^{k}\alpha_t\rho^{-\left\lfloor\frac{t-s}{T}\right\rfloor}b_s \leq \sum_{s=1}^{k}a_sb_s\sum_{t=0}^{T-1}\sum_{m=0}^{\infty}\rho^m \leq \frac{T}{1-\rho}\sum_{s=1}^{k}a_sb_s
\end{align*}
further we have
\begin{align*}
    &\sum_{t=1}^{k}a_t\left(\sum_{s=1}^{t}\rho^{-\left\lfloor\frac{t-s}{T}\right\rfloor}b_s\right)^2 = \sum_{t=1}^{k}a_t\sum_{s=1}^{t}\rho^{-\left\lfloor\frac{t-s}{T}\right\rfloor}b_s\sum_{r=1}^{t}\rho^{-\left\lfloor\frac{t-r}{T}\right\rfloor}b_r = \sum_{t=1}^{k}a_t\sum_{s=1}^{t}\sum_{r=1}^{t}\rho^{-\left\lfloor\frac{t-s}{T}\right\rfloor+\left\lfloor\frac{t-r}{T}\right\rfloor}b_sb_r\\
    \leq & \sum_{t=1}^{k}a_t\sum_{s=1}^{t}\sum_{r=1}^{t}\rho^{-\left\lfloor\frac{t-s}{T}\right\rfloor+\left\lfloor\frac{t-r}{T}\right\rfloor}\frac{b_s^2 + b_r^2}{2} = \sum_{t=1}^{k}a_t\sum_{s=1}^{t}\sum_{r=1}^{t}\rho^{-\left\lfloor\frac{t-s}{T}\right\rfloor+\left\lfloor\frac{t-r}{T}\right\rfloor}b_s^2\\
    \leq & \sum_{t=1}^{k}a_t\sum_{s=1}^{t}b_s^2\rho^{-\left\lfloor\frac{t-s}{T}\right\rfloor}\sum_{r=1}^{t}\rho^{-\left\lfloor\frac{t-r}{T}\right\rfloor}\leq \sum_{t=1}^{k}a_t\sum_{s=1}^{t}b_s^2\rho^{-\left\lfloor\frac{t-s}{T}\right\rfloor}\sum_{r=0}^{T-1}\sum_{m=0}^{\infty}\rho^m\\
    \leq & \frac{T}{1-\rho}\sum_{t=1}^{k}a_t\sum_{s=1}^{t}\rho^{-\left\lfloor\frac{t-s}{T}\right\rfloor}b_s^2 \overset{\text{Using $S_k$}}{\leq} \frac{T^2}{(1-\rho)^2}\sum_{s=1}^{k}a_sb_s^2
\end{align*}
That completes the proof.
\end{proof}